\documentclass{article} 
\usepackage{iclr2027_conference,times}

\usepackage{amsmath,amsfonts,bm}

\def\eqref#1{equation~\ref{#1}}

\def\1{\bm{1}}

\DeclareMathAlphabet{\mathsfit}{\encodingdefault}{\sfdefault}{m}{sl}
\SetMathAlphabet{\mathsfit}{bold}{\encodingdefault}{\sfdefault}{bx}{n}

\usepackage{hyperref}
\usepackage{url}
\usepackage{amsmath,amssymb,amsthm}
\usepackage{booktabs}
\usepackage{multirow}
\usepackage{graphicx}
\usepackage{algorithm}
\usepackage{algorithmic}
\usepackage{xcolor}
\usepackage{cleveref}

\newtheorem{proposition}{Proposition}
\newcommand{\method}{LatticeSMC}
\newcommand{\task}{budget-matched chunked steering}

\title{LatticeSMC: Where to Spend Inference-Time\\Compute in Chunked Sequence Generators}

\author{Xuanchen Wang \\
School of Computer Science\\
The University of Sydney\\
Sydney, NSW, Australia  \\
\texttt{xwan0579@uni.sydney.edu.au} \\
\And
Heng Wang \\
School of Computer Science \\
The University of Sydney \\
Sydney, NSW, Australia \\
\texttt{heng.wang@sydney.edu.au} \\
\AND
Weidong Cai \\
School of Computer Science \\
The University of Sydney\\
Sydney, NSW, Australia \\
\texttt{tom.cai@sydney.edu.au}
}

\iclrfinalcopy 
\begin{document}

\maketitle
\lhead{}  
\renewcommand{\headrulewidth}{0pt}

\begin{abstract}
Long-form generators for music, motion and video produce a sequence chunk by chunk, each chunk by iterative denoising, while the rewards that matter are defined on the whole sequence. Existing inference-time steering methods act on one axis of this process at a time, best-of-N at the end, Feynman-Kac steering across denoising steps, streaming pipelines across chunks, and have been compared under different return rules at unmatched compute. We introduce \task{}, the task of maximizing a sequence-level reward with a chunked generator under a fixed number of denoiser evaluations, and propose \method{}, a sampler derived from a Feynman-Kac model on the two-dimensional lattice of chunk index and denoising step. Two exact telescoping results make the design derivable: for chunk-additive rewards the two axes carry identical weights, so resampling belongs wherever lookahead is cheapest; for terminal rewards any prefix score yields an exact intermediate potential, so a prefix-evaluable reward is a twist with no estimation and no extra denoiser evaluations. \method{} resamples properly on these exact potentials, at chunk boundaries and, when the plug-in reward is free, within chunks, and returns either a draw from its particle approximation of the reward-tilted target or its best particle, and recovers Feynman-Kac steering as its within-chunk schedule. Under matched compute, on a music-to-dance diffusion model and a 40-second text-to-music model, it raises beat alignment from 0.234 to 0.441 (best-of-N: 0.354) and prompt adherence from 0.470 to 0.560 at 32 particles with held-out quality metrics at base, holds its lead on the long-range reward at six and eight chunks, and is preferred by human raters over every comparator in 60 to 77 percent of pairwise judgements. Two measurable properties of the reward settle what remains: how hard to commit at a boundary follows the information in its exact potential, which the effective-sample-size gate reads off and greedy pruning cannot, and whether content-axis lookahead pays follows the within-set predictability of future reward, measured before any particle is spent. The source code is available at \href{https://github.com/wangxuanchen88/LatticeSMC}{\texttt{LatticeSMC}}.
\end{abstract}

\section{Introduction}
\label{sec:intro}

Generative models for long-form sequential data increasingly produce their output in chunks: dance diffusion conditions each five-second window on the last \citep{tseng2023edge}, latent text-to-music models generate a window that inpainting extends past its edge \citep{evans2024longform,evans2024stableaudioopen}, and streaming video models denoise frame blocks against a growing context \citep{chen2024diffusionforcing,yin2025causvid,huang2025selfforcing}. Each chunk is generated by iterative denoising \citep{ho2020ddpm,lipman2023flow}, so a sequence is the result of a process with two clocks, the denoising step within a chunk and the chunk index, and the rewards that matter are rarely local to a chunk: a choreography must follow the beat for its whole duration and a song must match its prompt over forty seconds rather than its first ten. Spending inference-time compute to raise such rewards without retraining the generator is our setting.

Each existing steering tool acts on a single axis of this process. Best-of-N draws complete sequences and keeps the best \citep{brown2024monkeys,ma2025inference}, spending every evaluation before any evidence arrives. Feynman-Kac steering \citep{singhal2025fk,skreta2025fkc} resamples across denoising steps on plug-in rewards of partially denoised states, and twisted sequential Monte Carlo replaces the plug-in by a learned or refined twist \citep{wu2023tds,wang2026trustregion,cdm2026}; for a chunked generator either means resampling on estimates of the current chunk while the sequence-level reward is out of reach. Streaming video systems prune candidate chunks after each block by greedy top-$M$ selection \citep{streamt12026}. These methods have been compared under different return rules and at budgets that do not charge the decoder and reward evaluations lookahead requires, so where on the two-clock lattice to place resampling events has had no principled answer.

Our starting point is that every one of these methods is a choice of potentials and resampling events on one object, a Feynman-Kac model \citep{delmoral2004,chopin2020} on the lattice of chunk index $k$ and denoising step $\tau$ (Figure~\ref{fig:framework}): best-of-N at the terminal corner, Feynman-Kac steering along the denoising axis, chunk pruning along the chunk axis with degenerate resampling. Two properties follow by telescoping: for a reward that decomposes over chunks, the product of a chunk's denoising-axis potentials equals its chunk-axis potential exactly, so the two axes define the same tilted target with the same per-chunk weights and differ only in where resampling occurs; for a terminal reward, any prefix scores whose final value is the reward define exact intermediate potentials, so a prefix-evaluable reward is a twist with no estimation. \method{} is the sampler these results prescribe: exact chunk rewards or prefix-score increments as potentials, so its target is the reward-tilted distribution without a learned component; resampling gated by effective sample size \citep{kong1994ess} at chunk boundaries; a within-chunk schedule switched on only when plug-in evaluation is free, since axis equivalence says those events change the resampling times and not the weights; and two return modes, a weighted draw for sampling and the best particle for search.

We evaluate \method{} under matched denoiser evaluations on four sequence-level rewards across two modalities: beat alignment and a long-range repetition reward for a music-to-dance diffusion model on AIST++ \citep{li2021aistpp}, and prompt adherence \citep{wu2023clap} and motif recurrence for a 40-second text-to-music flow model. 

Our contributions can be summarized as follows:
\begin{itemize}
\item {}We introduce the lattice Feynman-Kac model, a unifying formulation under which existing inference-time methods correspond to different placements of potentials. We further derive two telescoping results: identical axis weights for chunk-additive rewards and exact prefix-score potentials for terminal rewards.

\item {}We propose \method{}, a training-free sampler whose every design decision follows directly from the model, with no learned components.

\item {}We derive practical criteria for commitment and lookahead. Commitment at each boundary should follow the information carried by its exact potential, which explains when the sharpest per-boundary tilt succeeds and when an ESS gate is required to preserve the particle set. Lookahead pays in proportion to the within-set predictability of future reward, a quantity measurable before sampling, which prefix-evaluable scores realize and learned regressors do not.

\item {}With 32 particles, \method{} matches or outperforms every compared method on every reward, raising beat alignment from 0.234 to 0.441 (best-of-N: 0.354; chunk pruning: 0.415) and prompt adherence from 0.470 to 0.560 while preserving held-out quality. It matches Feynman-Kac steering at a quarter of the decoder calls, retains its advantage on long-range rewards at six and eight chunks, holds its ranking from 4M to 38M parameters, and is preferred by human raters.
\end{itemize}

\begin{figure}[t]
\centering
\includegraphics[width=\textwidth]{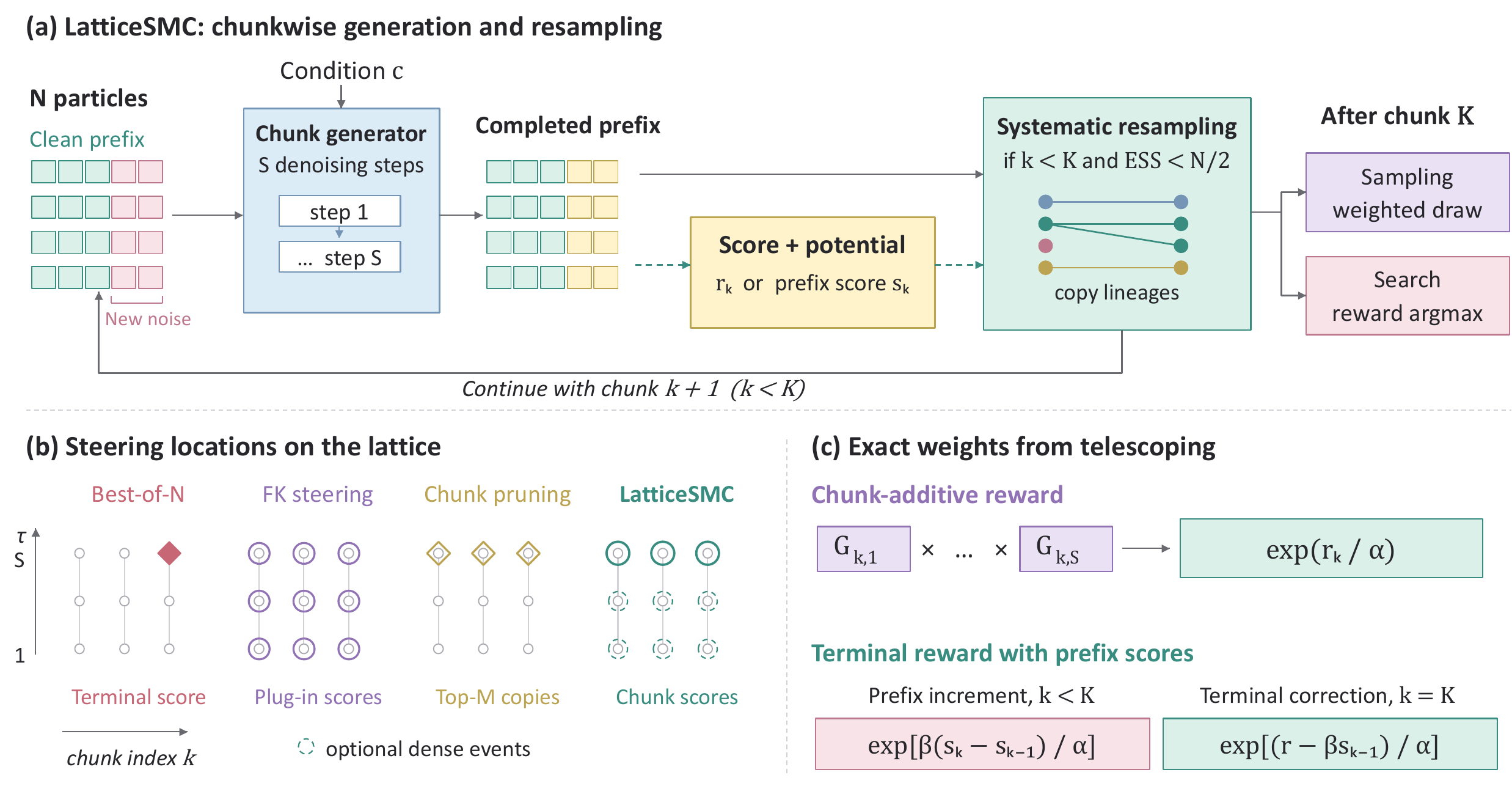}
\caption{\textbf{The lattice Feynman-Kac model and \method{}.} (a) One pass: $N$ particles extend their clean prefix by one chunk, each completed prefix is scored (chunk reward $r_k$ or prefix score $s_k$), the set is resampled systematically at the boundary when the ESS falls below $N/2$, and after the last chunk the sampler returns a weighted draw (sampling) or the best particle (search). (b) Where each method places its potentials on the lattice of chunk index $k$ and denoising step $\tau$. (c) The telescoping results: denoising-axis potentials multiply to the chunk potential (Proposition~\ref{prop:axis}); prefix-score increments with a terminal correction multiply to $\exp(r/\alpha)$ (Proposition~\ref{prop:prefix}).}
\label{fig:framework}
\vspace{-6pt}
\end{figure}

\section{Related Work}
\label{app:related}

\textbf{Inference-Time Compute and Search.} Inference-time compute improves language-model outputs through repeated sampling with verifiers \citep{brown2024monkeys}, process-level verification \citep{lightman2024verify}, tree search over partial solutions \citep{yao2023tot}, and adaptive allocation between sampling and revision \citep{snell2024scaling}; controlled decoding similarly scores partial sequences with learned values \citep{mudgal2024controlled}. For diffusion models, analogous methods search over noises or trajectories \citep{ma2025inference} or perform reward-guided test-time alignment \citep{kim2025das}. These approaches operate on a single generation clock. We instead treat chunk index and denoising step as two clocks of one process and study where inference-time compute should be allocated across this lattice.

\textbf{Reward-Guided Diffusion Sampling.} Existing approaches broadly modify the sampling kernel, fine-tune the generator, or reweight a sample population. Kernel-based guidance includes classifier guidance \citep{dhariwal2021diffusion}, classifier-free guidance \citep{ho2022cfg}, diffusion posterior sampling \citep{chung2023dps}, and loss-guided sampling \citep{song2023lgd}; see \citet{uehara2025review}. Fine-tuning methods optimize diffusion models with reinforcement learning \citep{black2024ddpo,fan2023dpok}, reward gradients \citep{clark2024draft}, or preference optimization \citep{wallace2024dpo}, but require reward-specific training and can over-optimize learned rewards \citep{gao2023overoptimization}. Population-based methods instead retain the kernel and use SMC during denoising \citep{trippe2023smcdiff,wu2023tds,dou2024fps,cardoso2024mcgdiff}, Feynman--Kac steering or correction \citep{singhal2025fk,skreta2025fkc}, and refined or amortized twists \citep{wang2026trustregion,cdm2026}. \method{} belongs to this family but introduces the chunk axis, showing when chunk- and denoising-axis formulations coincide, when prefix scores define exact twists, and where resampling should be placed at matched cost. Kernel modification is complementary to this placement problem.

\textbf{Sequential Monte Carlo.} \method{} builds on bootstrap filtering \citep{gordon1993bootstrap}, ESS-based adaptive resampling \citep{kong1994ess,liu1998smc}, and low-variance resampling \citep{douc2005resampling}. Auxiliary particle filters and SMC samplers generalize these constructions \citep{pitt1999apf,delmoral2006smcsamplers}, while twisting toward expected future weight is classical \citep{whiteley2014twisted,guarniero2017iapf,heng2020controlled}. The broader Feynman--Kac view is developed in \citet{delmoral2004,chopin2020,naesseth2019}. SMC has also been used for language generation with programmable potentials \citep{lew2023smc}, learned twists \citep{zhao2024twisted}, structured constraints \citep{loula2025smc}, and rollout-based values \citep{puri2025rollout}. Our analysis connects the usefulness of twisting to the within-set predictability of future reward and therefore applies beyond the specific generators studied here.

\textbf{Chunked and Streaming Generators.} Long-form generation often proceeds block by block. In video, diffusion forcing \citep{chen2024diffusionforcing}, rolling diffusion \citep{ruhe2024rolling}, and causal autoregressive distillation \citep{yin2025causvid,huang2025selfforcing} generate or refine successive windows; streaming test-time scaling further prunes candidate blocks by reward \citep{streamt12026}. Music models similarly extend sequences autoregressively \citep{copet2023musicgen} or through latent-diffusion inpainting \citep{evans2024longform,evans2024stableaudioopen}. In dance, generation has progressed from token-based models \citep{siyao2022bailando} to diffusion approaches \citep{tseng2023edge,alexanderson2023lda}, with long-form choreography assembled from chunks or primitives \citep{li2023finedance,li2024lodge}. These systems focus on how to generate long sequences; our question is instead where inference-time compute should be spent across the chunk and denoising axes.

\section{The Lattice Feynman-Kac Model}
\label{sec:method}

\subsection{Budget-matched chunked steering}
\label{sec:task}

A chunked generator produces a sequence $x = (x^{(1)}, \dots, x^{(K)})$ of $K$ chunks under a condition $c$. Chunk $k$ is generated by $S$ denoising steps from noise to a clean chunk, each conditioned on the clean prefix $x^{(<k)}$ and $c$, so the base model factorizes as $p(x \mid c) = \prod_{k} p(x^{(k)} \mid x^{(<k)}, c)$, each factor the marginal of $S$ transitions $M_{k,\tau}(x^{(k)}_{\tau} \mid x^{(k)}_{\tau-1}, x^{(<k)}, c)$ of the sampler's kernel from noise $x^{(k)}_{0}$ to the clean chunk $x^{(k)} = x^{(k)}_{S}$, where $x^{(k)}_{\tau}$ is the state of chunk $k$ after denoising step $\tau$. A reward $r$ on complete sequences is \emph{chunk-additive} if it is supplied as a sum $r(x) = \sum_{k} r_k(x^{(\le k)})$ of chunk terms computable from the prefix through chunk $k$ as the sequence is generated, and \emph{terminal} if only its value on the complete sequence is supplied, so that intermediate information must come from auxiliary prefix scores. For a temperature $\alpha > 0$ the reward-tilted target is $\pi_{\alpha}(x \mid c) \propto p(x \mid c)\, \exp(r(x)/\alpha)$. \emph{Budget-matched chunked steering} is the task of returning one sequence given the base model, the reward and a budget $B$ of denoiser evaluations (one per particle per denoising step), under one of two objectives, both reported: \emph{sampling}, where the returned sequence should be distributed as $\pi_{\alpha}$, and \emph{search}, where it should maximize $r$.

\subsection{The lattice}
\label{sec:lattice}

Generation visits the lattice $\mathcal{L} = \{(k,\tau) : 1 \le k \le K,\ 1 \le \tau \le S\}$ in order, and every inference-time method we consider is a placement of potentials on it (Figure~\ref{fig:framework}(b)). A Feynman-Kac model on $\mathcal{L}$ \citep{delmoral2004,chopin2020} attaches to every lattice point a nonnegative potential $G_{k,\tau}$, a function of the state $s_{k,\tau} = (x^{(<k)}, x^{(k)}_{\tau})$ and its lineage, and defines the unnormalized path measure $\gamma(x_{\mathcal{L}}) = p(x_{\mathcal{L}} \mid c)\, \prod_{(k,\tau)} G_{k,\tau}(s_{k,\tau})$ over trajectories of intermediate states. The model is \emph{valid} for the target if the product of its potentials along any completed trajectory equals $\exp(r(x)/\alpha)$ up to a constant, so that the marginal of $\gamma$ on the clean sequence is $\pi_{\alpha}$. Sequential Monte Carlo approximates $\gamma$ with $N$ particles that follow the base kernels, multiply their weights by $G_{k,\tau}$ at each lattice point, and are resampled when $\mathrm{ESS} = (\sum_i w_i)^2 / \sum_i w_i^2$ falls below a threshold \citep{kong1994ess,doucet2011tutorial}; it converges to $\pi_{\alpha}$ as $N \to \infty$ \citep{delmoral2004}.

Every existing method is a placement on this lattice (Table~\ref{tab:placements}, Appendix~\ref{app:placements}). Best-of-N places one potential $\exp(r(x)/\alpha)$ at the terminal corner $(K,S)$, importance resampling there being its sampling form and the argmax its search form. Feynman-Kac steering \citep{singhal2025fk} places potentials along the denoising axis, $G_{k,\tau} = \exp\big((\hat r_{k,\tau} - \hat r_{\mathrm{prev}})/\alpha\big)$, with $\hat r_{k,\tau}$ the reward of the plug-in clean prediction $\hat x^{(k)}_{0}(x^{(k)}_{\tau})$ appended to the prefix and $\hat r_{\mathrm{prev}}$ the previous value along the lineage. Chunk pruning in the manner of \citet{streamt12026} places potentials at $\tau = S$ and replaces resampling by replication of the top $M$ particles; at $M = 1$ it is the $\alpha \to 0$ limit of ESS-gated resampling with argmax return (unique maximum, $N \ge 3$). \method{} places proper potentials on the chunk axis, as the two results in Figure~\ref{fig:framework}(c) prescribe.

\subsection{Two telescoping results}
\label{sec:props}

\begin{proposition}[Axis equivalence for chunk-additive rewards]
\label{prop:axis}
Let $r(x) = \sum_{k=1}^{K} r_k(x^{(\le k)})$, let $\hat r_{k,\tau}$ be the accumulated reward through chunk $k-1$ plus the chunk term evaluated on the plug-in prediction at step $\tau$, so that $\hat r_{k,S} = \sum_{j<k} r_j + r_k(x^{(\le k)})$ because the plug-in at the final step is the chunk itself, and let the denoising-axis potentials of chunk $k$ be the difference potentials of Feynman-Kac steering at any subset of steps $\mathcal{T}_k \subseteq \{1,\dots,S\}$ that contains $S$. Then for every lineage:
\begin{equation}
\textstyle\prod_{\tau \in \mathcal{T}_k} G_{k,\tau} \;=\; \exp\!\big(r_k(x^{(\le k)})/\alpha\big),
\label{eq:telescope}
\end{equation}
the chunk-axis potential of chunk $k$. Consequently the denoising-axis and chunk-axis Feynman-Kac models have the same path measure and assign the same per-chunk weight to every lineage; their sequential Monte Carlo approximations differ only in the lattice points at which resampling may occur.
\end{proposition}

The proof (Appendix~\ref{app:proofs}) is a telescoping sum of log-potentials, the identity drawn in Figure~\ref{fig:framework}(c); resampling permutes whole lineages and leaves it unchanged. Since the weights are identical, the choice of axis for a chunk-additive reward is a choice of where to spend resampling, to be made on cost: within-chunk events are free when the plug-in reward is, and the boundary offers the same weight at the lowest cost otherwise.

\begin{proposition}[Prefix-score potentials for terminal rewards]
\label{prop:prefix}
Let $r$ be any reward on complete sequences and let $s_0, s_1, \dots, s_K$ be any functions such that $s_k$ depends only on the prefix $x^{(\le k)}$, $s_0$ is a constant, and $s_K = r$. For any $\beta \in [0,1]$ define the chunk-axis potentials:
\begin{equation}
G_k \;=\; \exp\!\Big(\tfrac{\beta}{\alpha}\,(s_k - s_{k-1})\Big) \ \ \text{for } 1 \le k < K,
\qquad
G_K \;=\; \exp\!\Big(\tfrac{1}{\alpha}\,(r - \beta\, s_{K-1})\Big).
\label{eq:prefix}
\end{equation}
Then $\prod_{k=1}^{K} G_k = \exp\big((r - \beta s_0)/\alpha\big) \propto \exp(r/\alpha)$, so the Feynman-Kac model is valid for $\pi_{\alpha}$ for every choice of prefix scores and every $\beta$; $\beta = 0$ is best-of-N importance sampling and $\beta = 1$ uses the full prefix-score increments.
\end{proposition}

In twisted sequential Monte Carlo \citep{guarniero2017iapf,heng2020controlled,zhao2024twisted} the ideal intermediate potential is the optimal twist $\psi_k(x^{(\le k)}) = \mathbb{E}_p[\exp(r(x)/\alpha) \mid x^{(\le k)}]$, which must be estimated. Proposition~\ref{prop:prefix} says that $\beta s_k / \alpha$ can play the role of $\log \psi_k$ for any prefix score without estimation and without changing the target: a poor score costs variance in the intermediate weights, never a wrong target; for a prefix-evaluable reward the natural choice $s_k = r(x^{(\le k)})$ has the accuracy of the reward on partial sequences.

\subsection{\method{}}
\label{sec:algorithm}

\method{} is the sequential Monte Carlo sampler for the chunk-axis Feynman-Kac model with the potentials of Propositions~\ref{prop:axis} and~\ref{prop:prefix} (Figure~\ref{fig:framework}(a); pseudocode in Appendix~\ref{app:algorithm}). \textbf{Potentials.} For a chunk-additive reward the potential at the end of chunk $k$ is $\exp(r_k(x^{(\le k)})/\alpha)$; for a terminal reward with a prefix-evaluable score it is the increment potential of Proposition~\ref{prop:prefix} with the terminal correction. \textbf{Resampling.} At the end of every chunk $k < K$ the weights are multiplied by the potential and, if $\mathrm{ESS} < N/2$, the set is resampled systematically \citep{doucet2011tutorial} and the weights reset; the gate keeps the set intact through chunks whose potential carries no information, such as the opening chunks of a motif-recurrence reward. When the plug-in reward is free, potentials can also sit at a fixed subset of denoising steps inside each chunk, with the same weights by Proposition~\ref{prop:axis} and earlier resampling times: the dense schedule. \textbf{Return modes.} After the last chunk the sampler returns a particle drawn in proportion to its weight (sampling mode: a draw from the particle approximation of $\pi_{\alpha}$, consistent as $N \to \infty$) or the highest-reward particle (search mode).

The cost is $N K S$ denoiser evaluations (times the guidance multiplicity), $K N$ reward evaluations and, on a latent generator, $K N$ decoder calls. Every other method is the same loop with a different set of lattice points at which the potential is not one and a different resampling operator, which is how all comparisons below are placed at equal budget; since a denoising-axis potential on a latent generator costs a decoder and a reward pass per particle per event, we also report model calls and wall-clock.

\section{Experiments and Results}
\label{sec:experiments}

\subsection{Setup}
\label{sec:setup}

\textbf{Generators} (Appendix~\ref{app:testbeds}). The first testbed is a music-to-dance diffusion model on AIST++ \citep{li2021aistpp} in the EDGE representation \citep{tseng2023edge}, made chunked by conditioning each 150-frame chunk on the clean previous one with the root integrated from the prefix, so the prefix is fixed and the seam continuous; it generates $K = 4$ chunks of 5 s ($K = 6$ and $8$ in Section~\ref{sec:costaxes}) with $S = 50$ DDIM steps \citep{song2021ddim} and classifier-free guidance, at 4M and 38M parameters. The second is a 40-second text-to-music flow model \citep{lipman2023flow} that extends a clip by inpainting past its edge with the prefix latents replaced at each of its $S = 8$ steps, $K = 4$ chunks of 10 s; each prefix score costs one decode and one CLAP \citep{wu2023clap} pass, which we count.

\textbf{Rewards.} On dance, \emph{beat alignment} is chunk-additive and chunk-local (kinematic beats against music beats), and \emph{repetition} scores chunk $k$ against chunk $k-2$, additive but long range, with no signal until the third chunk; it is a controlled long-range reward, not a quality target (real dancers score $0.412$, the base model $0.718$). On music, \emph{prompt adherence} is the CLAP score of the full clip against its prompt, terminal and evaluable on any prefix, and \emph{motif recurrence} the similarity of the last chunk to the first, with a prefix score at the last boundary. Held out from every steering decision are physical plausibility, realism and pose diversity on dance, and tempo adherence, spectral energy, aesthetics \citep{tjandra2025audiobox}, seam continuity and embedding diversity on music (Appendix~\ref{app:rewards}).

\textbf{Methods and Budget.} All methods run in the same loop with the same kernels and noise streams: best-of-N with importance draw and with argmax, chunk pruning in the manner of Stream-T1 \citep{streamt12026}, its top-$k$ retention isolated at $M = \max(1, N/4)$ over the common kernel (swept in Section~\ref{sec:costaxes}), learned-twist SMC on dance (Section~\ref{sec:lookahead}), and \method{} with the chunk potential on the dance rewards and the prefix-score potential on the music rewards, its dense schedule scoring at five (dance) or four (music) steps per chunk. We run \method{} with both schedules. Its dense schedule is the Feynman-Kac steering placement: on a chunk-additive reward the two are the same algorithm (Proposition~\ref{prop:axis}), and on the music rewards the plug-in CLAP differences telescope to the prefix-score increments of Proposition~\ref{prop:prefix} (products agreeing to $10^{-14}$ on every stored lineage), so that row is the strongest form of the baseline and our own schedule at once. Budgets are $N \in \{1, 2, 4, 8, 16, 32\}$ particles at $NKS$ denoiser evaluations times the guidance multiplicity (400$N$ on dance, 32$N$ on music), identical across methods. The temperature is set from base samples before any steered sample exists (Appendix~\ref{app:impl}): $\alpha = 0.0068$ and $0.019$ on the two music rewards, and the sharpest value of the sweep $\alpha \in \{0.2, 0.05, 0.02\}$ on dance, reported in full (Table~\ref{tab:rba_full}) and shared by every method. Results are means over 40 conditions with 4 seeds on dance and 2 on music, with 95\% percentile bootstrap intervals over conditions; method differences are paired bootstraps, and an asterisk marks an interval that excludes zero.

\begin{figure}[t]
\centering
\includegraphics[width=\textwidth]{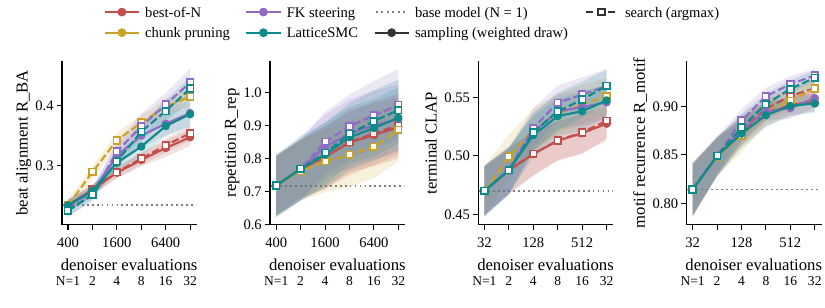}
\caption{\textbf{Reward against denoiser evaluations at matched budget.} Four rewards, $N \in \{1, \dots, 32\}$, means over 40 conditions with 95\% bootstrap intervals; solid: sampling rule (draw), dashed: search rule (argmax). \method{} uses the chunk potential on dance ($\alpha = 0.02$) and the prefix-score potential on music; FK steering is its dense schedule.}
\label{fig:curves}
\vspace{-6pt}
\end{figure}

\begin{table}[t]
\caption{Reward of the returned sequence at $N = 32$ particles and identical denoiser evaluations, under the sampling rule (draw) and the search rule (argmax), with 95\% bootstrap intervals for the base. A star marks the schedule of \method{} its cost model prescribes, dense where the plug-in reward is free (dance) and boundary where scoring costs a decode (music); the dense schedule is the Feynman-Kac steering placement (Proposition~\ref{prop:axis}), so that row is also the baseline. Chunk pruning is at $M = N/4$ (swept in Section~\ref{sec:costaxes}). Dash: return rule absent or not run; $^{a}$: additive scale, not ranked. Bold: best per column; underlined: second best; paired differences with intervals in Table~\ref{tab:intervals}.}
\label{tab:main}
\vspace{-4pt}
\begin{center}
\footnotesize
\setlength{\tabcolsep}{3.5pt}
\resizebox{\textwidth}{!}{\begin{tabular}{l cc cc cc cc}
\toprule
& \multicolumn{2}{c}{Beat alignment} & \multicolumn{2}{c}{Repetition} & \multicolumn{2}{c}{Prompt adherence} & \multicolumn{2}{c}{Motif recurrence} \\
& draw & argmax & draw & argmax & draw & argmax & draw & argmax \\
\midrule
Base ($N = 1$) & \multicolumn{2}{c}{0.234 [0.223, 0.243]} & \multicolumn{2}{c}{0.718 [0.628, 0.804]} & \multicolumn{2}{c}{0.470 [0.449, 0.491]} & \multicolumn{2}{c}{0.814 [0.786, 0.839]} \\
Best-of-N & 0.347 & 0.354 & 0.895 & 0.901 & 0.527 & 0.530 & \underline{0.904} & 0.918 \\
Chunk pruning (top-$M$) & -- & 0.415 & -- & 0.887 & -- & 0.551 & -- & 0.918 \\
Twisted SMC, learned twist & \underline{0.385} & 0.419$^{a}$ & 0.871 & 0.895 & -- & -- & -- & -- \\
\method{}, dense (FK steering placement) & \textbf{0.387}$^{\star}$ & \textbf{0.441}$^{\star}$ & \underline{0.921}$^{\star}$ & \textbf{0.962}$^{\star}$ & \underline{0.546} & \underline{0.559} & \textbf{0.908} & \textbf{0.932} \\
\method{}, boundary & \underline{0.385} & \underline{0.428} & \textbf{0.924} & \underline{0.947} & \textbf{0.547}$^{\star}$ & \textbf{0.560}$^{\star}$ & 0.903$^{\star}$ & \underline{0.929}$^{\star}$ \\
\bottomrule
\end{tabular}}
\end{center}
\end{table}

\subsection{Chunk-axis resampling is the compute-efficient placement}
\label{sec:mainresult}

Figure~\ref{fig:curves} plots reward against denoiser evaluations, Table~\ref{tab:main} the $N = 32$ endpoint under both return rules, and Tables~\ref{tab:rba_full} to~\ref{tab:motif_full} every budget. Three results hold across every reward, modality and model size.

First, selecting at chunk boundaries beats selecting at the end: under the draw rule the chunk potentials raise the reward over terminal importance sampling by $0.038$ on beat alignment, $0.028$ on repetition and $0.020$ on prompt adherence at $N = 32$ (intervals exclude zero), and under the search rule the gaps are larger: with its prescribed schedule \method{} reaches $0.441$ on beat alignment against $0.354$ for best-of-N argmax, $0.962$ against $0.901$ on repetition, $0.560$ against $0.530$ on prompt adherence and $0.929$ against $0.918$ on motif recurrence. On dance that schedule is the Feynman-Kac steering placement; the lattice view adds the exactness proof, the matched-compute comparison under one return rule, and the boundary schedule, which matches it under the sampling rule at a fifth of the reward evaluations.

Second, under a common return rule, ESS-gated resampling beats top-$M$ replication at $M = N/4$ at the largest budgets on every reward, and at every $M$ on the long-range rewards (Section~\ref{sec:costaxes}). With argmax return the boundary schedule alone is ahead of chunk pruning at $N = 32$ by $0.013$ on beat alignment, $0.059$ on repetition, $0.009$ on prompt adherence and $0.011$ on motif recurrence (intervals exclude zero), and the prescribed dense schedule by $0.026$ and $0.075$ on dance. On repetition it is ahead under the draw rule as well, by $0.036$, because pruning replicates blind for two chunks while the gate keeps the full set until the reward arrives. The two schedules are Proposition~\ref{prop:axis} observed (Table~\ref{tab:main}; intervals in Table~\ref{tab:intervals}): within $0.006$ under the sampling rule at $N = 32$ on the 4M model and on music, within $0.024$ at 38M and at every budget; under the search rule the dense schedule is ahead on dance by $0.013^{*}$ on beat alignment at 4M ($0.017^{*}$ at 38M) and by $0.016$ on repetition (interval containing zero at 4M, $0.051^{*}$ at 38M), the value of earlier resampling when scoring is free, and within $0.006$ on music (intervals containing zero), where every dense event costs a decode.

Third, the held-out check shows no sign of reward exploitation (Table~\ref{tab:heldout}; repetition and motif runs in Tables~\ref{tab:prunesweep}, \ref{tab:k6} and~\ref{tab:k8}): on dance every held-out quantity stays inside the base model's bootstrap interval for every method; on music tempo adherence, aesthetics and seam continuity stay within 6 percent of base, while high-frequency energy falls by up to 12 percent and embedding diversity by up to 21 percent under every steering method, as selection among samples must.

\subsection{The return rule is worth as much as the resampling scheme}
\label{sec:returnrule}

\begin{table}[t]
\caption{The return rule and the crossover for the boundary schedule of \method{}, so that the crossover measures resampling on chunk potentials alone: paired difference of \method{} (argmax) minus chunk pruning by particle count (asterisk: 95\% interval excludes zero; dance values at $N \ne 8, 32$ on the additive logged scale) and the gain of argmax over draw at $N = 32$; the 38M model reproduces every crossover (Figure~\ref{fig:crossover}).}
\label{tab:crossover}
\vspace{-4pt}
\begin{center}
\footnotesize
\setlength{\tabcolsep}{4pt}
\resizebox{\textwidth}{!}{\begin{tabular}{l ccccc c}
\toprule
\method{} (argmax) $-$ pruning at $N =$ & 2 & 4 & 8 & 16 & 32 & argmax - draw \\
\midrule
Beat alignment & $-0.032^{*}$ & $-0.029^{*}$ & $-0.015^{*}$ & $-0.002$ & $+0.013^{*}$ & $+0.042^{*}$ \\
Repetition & $+0.007$ & $+0.026^{*}$ & $+0.063^{*}$ & $+0.078^{*}$ & $+0.059^{*}$ & $+0.023^{*}$ \\
Prompt adherence & $-0.012^{*}$ & $-0.002$ & $+0.002$ & $+0.004^{*}$ & $+0.009^{*}$ & $+0.012^{*}$ \\
Motif recurrence & $+0.001$ & $+0.008$ & $+0.007^{*}$ & $+0.012^{*}$ & $+0.011^{*}$ & $+0.026^{*}$ \\
\bottomrule
\end{tabular}}
\end{center}
\end{table}

Each proper method returns a draw from its normalized weights, the particle approximation of $\pi_{\alpha}$, or the best of the same particles; switching from the first to the second adds $0.012$ to $0.042$ at $N = 32$ (Table~\ref{tab:crossover}, last column), and the gain grows with $N$. The comparison with greedy pruning, which always returns its best, therefore has to be made under one rule; Table~\ref{tab:crossover} makes it by $N$ (Figure~\ref{fig:crossover}). Pruning leads at the smallest budgets on beat alignment and prompt adherence; \method{} with argmax return crosses it by $N = 32$ on every reward ($N = 4$ on motif recurrence, $N = 8$ on prompt adherence) and leads from $N = 4$ on repetition under both rules.

\subsection{Tilt, horizon, temperature and cost}
\label{sec:costaxes}

\begin{table}[t]
\caption{Horizon and per-chunk tilt, 4M dance model, $N = 32$, argmax return, $\alpha = 0.02$ unless stated. $K = 4$ and $6$ use the 40 conditions, $K = 8$ the 16 with a long enough recording; repetition is the sum over the $K - 2$ chunk pairs, beat alignment a mean over beats; $\alpha = 0.01$ is the prespecified second point, equal to the per-boundary tilt $\alpha \propto 1/K$ at $K = 8$, run on beat alignment, where the fixed-$\alpha$ ordering reversed. Best per column in bold, second best underlined; dash: not run. Paired differences, resampling rates, base and draw rows: Appendix~\ref{app:horizon}.}
\label{tab:horizon}
\vspace{-4pt}
\begin{center}
\footnotesize
\setlength{\tabcolsep}{4pt}
\begin{tabular}{l ccc ccc}
\toprule
& \multicolumn{3}{c}{Beat alignment} & \multicolumn{3}{c}{Repetition} \\
$K =$ & 4 & 6 & 8 & 4 & 6 & 8 \\
\midrule
Best-of-N & 0.354 & 0.333 & 0.335 & 0.901 & 1.795 & 2.544 \\
Chunk pruning ($M = N/4$) & 0.415 & 0.410 & 0.434 & 0.887 & \underline{1.830} & 2.555 \\
\method{}, boundary & \underline{0.428} & 0.393 & 0.382 & \underline{0.947} & \textbf{1.930} & \underline{2.676} \\
\method{}, dense & \textbf{0.441} & 0.400 & 0.396 & \textbf{0.962} & \textbf{1.930} & \textbf{2.711} \\
\method{}, boundary / dense, $\alpha = 0.01$ & -- & \underline{0.436} / \textbf{0.441} & \textbf{0.455} / \underline{0.450} & -- & -- & -- \\
\bottomrule
\end{tabular}
\end{center}
\end{table}

\textbf{Tilt.} How sharply the potentials tilt is the second dial of a placement, and the lattice fixes its far end: chunk pruning with $M = 1$, one lineage kept at every boundary, is the $\alpha \to 0$ limit of the boundary schedule with argmax return when the maximum is unique. Sweeping $M \in \{1, N/4, N/2\}$ at $N = 32$ on all four rewards (Table~\ref{tab:prunesweep}) continues the temperature sweep of Figure~\ref{fig:temperature} to its corner, and the outcome follows the information in the potentials. On the two rewards whose potential carries the reward at every boundary, the sharpest tilt is the best setting in the study under the search objective, where $\alpha$ is free: $\alpha \to 0$ reaches $0.473$ on beat alignment ($+0.045^{*}$ over the boundary argmax at $\alpha = 0.02$) and $0.564$ on prompt adherence ($+0.004^{*}$), with the held-out quality metrics at base and the lowest seed diversity of the sweep. On the two rewards that arrive late, whose potentials are uninformative at the early boundaries, the ordering reverses: $\alpha \to 0$ is the worst pruning variant, $M = N/2$ the best, and every $M$ trails the ESS-gated boundary schedule, by $0.031$ to $0.096$ on repetition and $0.006$ to $0.020$ on motif recurrence (intervals exclude zero). The ESS gate follows that information at any $\alpha$, since equal potentials leave the set intact; top-$M$ replication commits at every boundary and so needs its $M$ chosen per reward.

\textbf{Horizon.} Table~\ref{tab:horizon} extends the dance testbed to $K = 6$ (30 s, all 40 conditions) and $K = 8$ (40 s, the 16 conditions with a long enough recording). On the long-range reward, a sum over the $K - 2$ chunk pairs, the boundary schedule leads pruning at every horizon, by $+0.100^{*}$ at $K = 6$ and $+0.121^{*}$ at $K = 8$ on the sum, which is $0.025$ and $0.020$ per pair against $0.030$ at $K = 4$, and leads best-of-N by $+0.135^{*}$ at $K = 6$ and $+0.133$ at $K = 8$ (interval over 16 conditions containing zero); the dense schedule ties it. On beat alignment the lead holds at every horizon once the tilt is set per boundary: at $\alpha = 0.01$, the prespecified second point, the boundary argmax leads pruning by $+0.026^{*}$ at $K = 6$ and $+0.021^{*}$ at $K = 8$, both schedules lead best-of-N by $0.10$ to $0.12$, and every held-out quantity lies inside the base interval on the same conditions. The rows at the $K = 4$ value $\alpha = 0.02$ are the ablation that says why: beat alignment is a mean over the beats of the whole sequence, so chunk $k$'s term is its share of that mean, with a range that shrinks as $1/K$, and at fixed $\alpha$ so does the per-boundary tilt $r_k/\alpha$; the ESS gate then fires at 64 and 50 percent of the boundaries instead of 91, and pruning, which commits whatever $\alpha$ is, moves ahead by $0.017^{*}$ and $0.052^{*}$ (both leave the base foot-contact interval at $K = 8$, Appendix~\ref{app:horizon}). For a reward normalized over the sequence, the tilt to hold fixed across horizons is $\alpha \propto 1/K$.

\textbf{Temperature, Scale and Cost.} The temperature sweep (Figure~\ref{fig:temperature}) shows the two schedules coinciding at every $\alpha$; a 38M dance model trained with the same recipe leaves every pairwise ordering and the crossover against pruning unchanged (Table~\ref{tab:scale}). On dance the plug-in reward is kinematic and the three cost axes coincide; on music a within-chunk event costs a decode and a CLAP pass, so at $N = 32$ the dense schedule makes four times the boundary schedule's decoder calls (512 against 128) at $3.6$ times its wall-clock for the same reward (Figure~\ref{fig:cost}).

\subsection{Human evaluation}
\label{sec:human}

To test whether a higher reward is perceptible and costs no quality that metrics miss, we ran a pairwise preference study on the two rewards with a direct perceptual reading, beat alignment and prompt adherence (Appendix~\ref{app:human}). For each of the 40 conditions of each testbed, five Prolific raters compared the sequence returned by \method{} (prescribed schedule, argmax, $N = 32$) against the base model, best-of-N argmax and chunk pruning, with side and order randomized and attention checks, on two questions: which follows the beat or matches the description better, and which they would rather watch or hear (Table~\ref{tab:human}). On the reward question \method{} is preferred over the base model in 73.5 percent of dance and 77.0 percent of music judgements, over best-of-N in 66.0 and 69.0, and over chunk pruning in 69.0 and 71.5; on the overall question its win rate against every comparator is 60.5 to 69.0 percent; all twelve intervals lie above 50 percent. Per-condition reward differences and win rates correlate positively for every comparator (Spearman $\rho$ $0.23$ to $0.67$): the reward gain is perceptible and does not cost perceived quality.

\section{Analysis: When Content-Axis Lookahead Pays}
\label{sec:lookahead}

\subsection{What a twist can know is a within-set quantity, and so is its accuracy}
\label{sec:predictability}

The experiments above use lookahead only where the reward supplies a prefix score; this section asks when a twist, the estimated future reward of twisted sequential Monte Carlo \citep{guarniero2017iapf,heng2020controlled,zhao2024twisted}, can improve resampling at all (Figure~\ref{fig:lookahead}, Appendix~\ref{app:lookahead}). Write $S_k = \sum_{j > k} r_j$ for the reward still to come after chunk $k$ ($S_k = r$ for a terminal reward). The optimal chunk-axis twist is $\psi_k(x^{(\le k)}) = \mathbb{E}_p[\exp(S_k/\alpha) \mid x^{(\le k)}, c]$, and it enters the sampler through the ratio $\psi_k / \psi_{k-1}$ multiplied into weights that are then normalized over the $N$ particles. Every particle in a set shares $c$, so any component of $\log \psi_k$ that depends on $c$ alone cancels in the normalized weights, and the base model already realizes the variation of $S_k$ the prefix does not determine; a twist can help only to the extent that the prefix determines the future reward within a set (Figure~\ref{fig:lookahead}(a)): for prefixes drawn under one condition, each continued $M$ times:
\begin{equation}
\mathrm{Var}[S_k \mid c] \;=\; \underbrace{\mathrm{Var}_{x^{(\le k)}}\big[\mathbb{E}[S_k \mid x^{(\le k)}, c]\big]}_{\text{between-prefix spread: ranking signal}} \;+\; \underbrace{\mathbb{E}_{x^{(\le k)}}\big[\mathrm{Var}[S_k \mid x^{(\le k)}, c]\big]}_{\text{within-prefix spread: rollout noise}},
\label{eq:decomp}
\end{equation}
and the first term's share of the total is the \emph{within-set predictability} $\rho_k$, estimated from base-model rollouts (Appendix~\ref{app:decomp}). It is a diagnostic on the conditional mean of $S_k$, while the twist is an exponential moment, so it is paired with the same decomposition of $\exp(S_k/\alpha)$ (Appendix~\ref{app:decomp}) and with direct within-set measurements of $\log\psi_k$ (Appendix~\ref{app:withinset}). For beat alignment the between-prefix share of $\exp(S_k/\alpha)$ is at most $0.14$ even pooled across conditions, so the optimal twist is close to a constant. For prompt adherence the prefix explains $0.47$, $0.62$ and $0.66$ of the terminal variance after one, two and three chunks. For repetition the pooled share is $0.9$, but within a set of same-condition particles the rollout target spreads by $2.9$ nats, about $2$ of them Monte Carlo error at $M = 4$ (Appendix~\ref{app:twistvariants}), so the within-set share is far below the pooled value.

\begin{table}[t]
\caption{Twists on the repetition reward at $\alpha = 0.02$, $N = 32$ (4M unless stated): cross-prompt fit, within-set Pearson correlation with the oracle after chunk one, root-mean-square residual against the rollout target on 500 fresh prefixes in nats, and effect relative to $\psi \equiv 1$ with paired intervals; the oracle is not budget-matched; dash: not measured.}
\label{tab:twist}
\vspace{-4pt}
\begin{center}
\footnotesize
\setlength{\tabcolsep}{4pt}
\resizebox{\textwidth}{!}{\begin{tabular}{lcccc}
\toprule
Twist & cross-prompt $R^2$ & within-set corr. & residual (nats) & effect at $N = 32$ \\
\midrule
$\psi \equiv 1$ (exact potentials only) & -- & -- & -- & $0$ \\
Learned regressor & $0.90$ & $-0.04$ & $4.3$ & $-0.053\ [-0.083, -0.028]$ \\
Learned regressor, 38M model & $0.91$ & -- & $3.3$ & $-0.018\ [-0.032, -0.006]$ \\
Regressor trained on within-set weights & -- & $0.01$ & -- & $-0.014\ [-0.033, -0.002]$ \\
Oracle, 16 rollouts per particle & -- & $1$ & $\approx 0.6$ & $+0.058\ [+0.034, +0.083]$ \\
\bottomrule
\end{tabular}}
\end{center}
\end{table}

The same cancellation says how a twist must be evaluated. On the repetition reward a regressor trained on rollouts predicts realized future reward across held-out prefixes with a cross-prompt $R^2$ of $0.90$, a number dominated by between-condition structure that cancels inside a set; on the quantity resampling consumes, the same regressor ranks the particles of a set almost at random, with a within-set correlation to an oracle twist on the same particles of $-0.04$ after the first chunk and $0.35$ after the second (Figure~\ref{fig:lookahead}(b)). Weights are exponentials of these values, so its residual of $4.3$ nats on fresh prefixes ($6.0$ within a set) is a weight error of a factor $10$ to $900$ on an uninformative ranking, and resampling on it gives the outcome in Table~\ref{tab:twist}: a twist with $0.90$ cross-prompt fit lowers the reward by $0.05$ at $N = 32$ relative to the exact-potential sampler ($\psi \equiv 1$) on the 4M model and by $0.02$ on the 38M model. Tempering, ensembling or shrinking the twist converge back onto the $\psi \equiv 1$ sampler (Appendix~\ref{app:twistvariants}): an absent ranking cannot be recovered by reweighting, so the numbers to design against are the within-set correlation with realized future reward and the weight-level error in nats.

\subsection{A design rule}
\label{sec:sources}

A twist has three sources (Figure~\ref{fig:lookahead}(c), Appendix~\ref{app:lookahead}): a prefix-evaluable reward, exact by Proposition~\ref{prop:prefix} at no extra denoiser cost ($+0.020$ terminal CLAP at $N = 32$); rollouts, unbiased but charged, which set the headroom ($+0.058$ on repetition) yet lose to additional particles at matched cost for $M \le 4$; and a learned predictor, whose rollout-derived targets carry about 2 nats of Monte Carlo error against a within-set spread of about 3, so that none we trained recovered a within-set ranking. Resample on the exact potentials with ESS gating, so that commitment follows the information each boundary carries, and set the tilt by the objective: under search, as sharp as the reward's arrival allows, up to the $\alpha \to 0$ corner for a chunk-local reward; under sampling, $\alpha$ defines the target; in both, hold it fixed per boundary across horizons. Then measure $\rho_k$ from base-model rollouts before adding lookahead: if it is small, add none; if the reward can be evaluated on a prefix, use the prefix score at full strength; if not, validate any estimated twist by its within-set correlation and residual in nats, never by its cross-prompt fit.

\section{DISCUSSION AND CONCLUSION}
\label{sec:conclusion}

Inference-time steering of a chunked generator is one Feynman-Kac model on the lattice of chunk index and denoising step, and two telescoping identities fix its target: identical axis weights for chunk-additive rewards, and exact prefix-score twists for terminal rewards. \method{}, the sampler they prescribe, ties or leads every method of the main comparison at matched compute across four rewards, two model sizes, and three horizons, while human evaluators also prefer its outputs. The remaining design choices are largely determined by two properties of the reward. First, commit at each boundary in proportion to the information carried by its exact potential, which ESS gating can read off adaptively and greedy pruning cannot. Second, add lookahead only in proportion to the within-set predictability of the future reward, so that extra compute is spent only when future scores are sufficiently informative for the current candidate set.

\textbf{Limitations.} The horizon study is on dance, since the music generator's 40 s context does not carry the prompted tempo over 80 s (Appendix~\ref{app:horizon}); the dance generators are small models trained for this study; the human study covers two rewards; gradient guidance, which composes with every placement, is not compared; and the propositions fix the target, not finite-particle behaviour, so a bound relating $\rho_k$ to the incremental-weight variance is open.

\clearpage

\bibliography{main}
\bibliographystyle{iclr2027_conference}

\newpage
\appendix

\begin{center}
    {\fontsize{17}{22}\selectfont APPENDIX}
    \vspace{1.5em}
\end{center}

\section{The Lattice Model: Proofs and Details}
\label{app:method}

\subsection{Proofs}
\label{app:proofs}

\begin{proof}[Proof of Proposition~\ref{prop:axis}]
Fix a lineage and a chunk $k$, and write the steps of $\mathcal{T}_k$ in increasing order as $\tau_1 < \dots < \tau_m = S$. The difference potentials are $G_{k,\tau_j} = \exp\big((\hat r_{k,\tau_j} - \hat r_{k,\tau_{j-1}})/\alpha\big)$ for $j \ge 2$ and $G_{k,\tau_1} = \exp\big((\hat r_{k,\tau_1} - r_{\mathrm{prefix}})/\alpha\big)$, where $r_{\mathrm{prefix}} = \sum_{j<k} r_j(x^{(\le j)})$ is the lineage's accumulated reward through chunk $k-1$ and every $\hat r$ is the additive reward of the prefix plus the chunk term evaluated on the plug-in prediction at that step. Summing the log-potentials telescopes to $(\hat r_{k,S} - r_{\mathrm{prefix}})/\alpha$. At the final step the plug-in prediction is the sampled chunk itself, so $\hat r_{k,S} = r_{\mathrm{prefix}} + r_k(x^{(\le k)})$ and the sum equals $r_k(x^{(\le k)})/\alpha$, which is \Cref{eq:telescope}. Resampling copies whole lineages together with their accumulated $\hat r$ values, so the identity holds along every lineage regardless of the resampling events. The product over $k$ is $\exp(r(x)/\alpha)$, so both models are valid for $\pi_{\alpha}$ with the same path measure; they differ only in the set of lattice points at which the potential is not one, which is the set of admissible resampling times.
\end{proof}

\begin{proof}[Proof of Proposition~\ref{prop:prefix}]
The log-potentials of \Cref{eq:prefix} sum to $\frac{\beta}{\alpha}\sum_{k=1}^{K-1}(s_k - s_{k-1}) + \frac{1}{\alpha}(r - \beta s_{K-1}) = \frac{1}{\alpha}(r - \beta s_0)$, and $s_0$ is a constant, so the product of potentials is $\exp(r/\alpha)$ up to a constant along every trajectory and the Feynman-Kac model is valid for $\pi_{\alpha}$. Setting $\beta = 0$ makes every intermediate potential one and the terminal potential $\exp(r/\alpha)$, which is best-of-N importance sampling.
\end{proof}

\subsection{The plug-in potential and the optimal twist}
\label{app:jensen}

Along the denoising axis, for a chunk-additive reward, the optimal twist at step $\tau$ of chunk $k$ for the current chunk's term is $\psi_{k,\tau}(s_{k,\tau}) = \mathbb{E}_p\big[\exp(r_k/\alpha) \mid s_{k,\tau}\big]$. Feynman-Kac steering replaces it by the plug-in $\exp\big(\hat r_{k,\tau}/\alpha\big)$ with $\hat r$ evaluated on the posterior-mean prediction $\hat x_0$. Since $u \mapsto \exp(u/\alpha)$ is convex, Jensen's inequality gives $\psi_{k,\tau} \ge \exp\big(\mathbb{E}[r_k \mid s_{k,\tau}]/\alpha\big)$, and the plug-in approximates the right-hand side up to the further gap between $r_k(\hat x_0)$ and $\mathbb{E}[r_k \mid s_{k,\tau}]$ for a nonlinear reward. Both gaps grow with the conditional variance of the current chunk given the noisy state, which is largest at the early denoising steps and vanishes at $\tau = S$, where the prediction is the sample. This is the sense in which denoising-axis lookahead on the current chunk is accurate late and uninformative early; Proposition~\ref{prop:axis} guarantees that whatever the plug-in's accuracy at intermediate steps, the chunk's weight at its end is exact, which fixes the target and not the finite-particle behaviour of the intermediate resampling events.

\subsection{Algorithm}
\label{app:algorithm}

\begin{algorithm}[h]
\caption{\method{}}
\label{alg:latticesmc}
\begin{algorithmic}[1]
\REQUIRE base kernels $M_{k,\tau}$, condition $c$, reward $r$ (chunk terms $r_k$ or prefix scores $s_k$), temperature $\alpha$, particles $N$, schedule $\in \{\text{boundary}, \text{dense}\}$, return mode $\in \{\text{draw}, \text{argmax}\}$
\STATE initialize $N$ empty prefixes, log-weights $\ell_i \gets 0$
\FOR{$k = 1, \dots, K$}
    \STATE $a_i \gets 0$ \hfill (log-potential already applied within chunk $k$; copied with its lineage on resampling)
    \FOR{$\tau = 1, \dots, S$}
        \STATE advance every particle one denoising step with $M_{k,\tau}$
        \IF{schedule is dense and $\tau \in \mathcal{T}_k \setminus \{S\}$}
            \STATE $\delta_i \gets \big(\hat r_{k,\tau}^{(i)} - \hat r_{\mathrm{prev}}^{(i)}\big)/\alpha$;\ \ $\ell_i \mathrel{+}= \delta_i$;\ \ $a_i \mathrel{+}= \delta_i$;\ \ \textbf{if} $\mathrm{ESS} < N/2$ \textbf{then} resample systematically, $\ell_i \gets 0$
        \ENDIF
    \ENDFOR
    \STATE $\ell_i \mathrel{+}= \log G_k^{(i)} - a_i$ \hfill (remaining part of the chunk term $r_k/\alpha$, or of the prefix-score increment \eqref{eq:prefix})
    \IF{$k < K$ and $\mathrm{ESS} < N/2$} \STATE resample systematically, $\ell_i \gets 0$ \ENDIF
\ENDFOR
\RETURN $x^{(i^\star)}$ with $i^\star \sim \mathrm{Cat}(\mathrm{softmax}(\ell))$ (draw) or $i^\star = \arg\max_i r(x^{(i)})$ (argmax)
\end{algorithmic}
\end{algorithm}

\subsection{Existing methods as lattice placements}
\label{app:placements}

\begin{table}[h]
\caption{Inference-time steering methods as placements on the lattice. IS: importance resampling at the terminal corner; SYS: systematic resampling gated by $\mathrm{ESS} < N/2$; TOP-$M$: deterministic replication of the top $M$ particles.}
\label{tab:placements}
\begin{center}
\footnotesize
\resizebox{\textwidth}{!}{\begin{tabular}{lllll}
\toprule
Method & Lattice points with $G \neq 1$ & Potential & Resampling & Return \\
\midrule
Best-of-N & $(K, S)$ & $\exp(r/\alpha)$ & IS or none & draw or argmax \\
Feynman-Kac steering & $(k, \tau)$, $\tau \in \mathcal{T}_k$ & plug-in difference & SYS & draw \\
Chunk pruning & $(k, S)$ & prefix reward, $\alpha \to 0$ & TOP-$M$ & argmax \\
\method{}, boundary & $(k, S)$ & exact chunk term or prefix increment & SYS & draw or argmax \\
\method{}, dense & $(k, \tau)$ and $(k, S)$ & as above, plus plug-in differences & SYS & draw or argmax \\
\bottomrule
\end{tabular}}
\end{center}
\end{table}

Best-of-N with importance resampling and best-of-N with argmax share their particle set and differ only in the final index. Feynman-Kac steering uses the difference potentials of Proposition~\ref{prop:axis} at a fixed subset of denoising steps and systematic resampling below $\mathrm{ESS} = N/2$; it does not resample at chunk boundaries, and its lineages carry their prefixes across chunks. Chunk pruning keeps the top $M = \max(1, \lfloor N/4 \rfloor)$ particles by prefix reward at every chunk boundary and expands each to $N/M$ continuations, which is deterministic replication rather than resampling; its return is the argmax. \method{} with $\psi \equiv 1$ is the sampler with the exact chunk potentials of Proposition~\ref{prop:axis} and no other intermediate potential; with a prefix score it is the sampler of Proposition~\ref{prop:prefix}. Every method is run inside the same loop with the same kernels and the same per-slot noise streams.

\subsection{Implementation details}
\label{app:impl}

\textbf{Noise Streams and Reproducibility.} Each particle slot $i$ of chunk $k$ draws its denoising noise from a generator seeded by $(\text{seed}, k, i)$, and a resampled particle continues with the stream of the slot it occupies. Every resampling decision and the final draw consume a separate generator seeded by $(\text{seed}, \text{method})$. Two runs of any method with the same seed produce bitwise-identical particle systems, and a method whose potentials are all one reproduces the base sampler bitwise, which is the check we use to validate every implementation.

\textbf{Effective Sample Size and Resampling.} With normalized weights $\bar w_i = w_i / \sum_j w_j$, $\mathrm{ESS} = 1/\sum_i \bar w_i^2$. Systematic resampling draws a single uniform $u \sim \mathcal{U}[0,1/N)$ and selects the particles whose cumulative weight intervals contain $u + j/N$ for $j = 0,\dots,N-1$; it preserves the expected count of every particle and, in practice though not in general, has lower variance than multinomial resampling \citep{douc2005resampling}.

\textbf{Cost Accounting.} The denoiser-evaluation counter is incremented by the number of rows in every denoiser forward call, so classifier-free guidance implemented as a doubled batch is charged twice per particle per step and guidance implemented as a single conditioned pass is charged once, matching the generator's actual cost. Decoder calls and reward evaluations are separate counters. For every reported sequence the counters are audited against the closed-form value $N K S$ times the guidance multiplicity.

\textbf{Hardware, Software and Timing.} All runs used two NVIDIA RTX A6000 (48 GB) GPUs in one machine. The dance harness runs PyTorch 2.14 (CUDA 13.0) in float32 with TF32 off, deterministic algorithms and cuDNN determinism on, the math attention kernel only, and all noise drawn from CPU generators; the music harness runs PyTorch 2.7.1 (CUDA 12.6) in the checkpoint's parameter precisions with a decoder seed pinned per (prompt, seed). Bitwise reproduction was verified by regenerating the 160 base sequences of each dance model three times, once with every sequence moved to the other physical GPU (maximum absolute difference 0), by the argmax regenerations, which reproduce the stored particle sets in 160 of 160 dance and 320 of 320 music replays, and by the music latents and waveforms across processes; a chunk generated inside a batch of $N$ particles differs from the same chunk generated alone by up to $6 \times 10^{-5}$, so reproduction holds for a fixed (method, $N$). Wall-clock per sequence is the generating process's own timer, averaged over the 80 sequences of the $N = 32$ row, with three processes per GPU sharing the two GPUs and the $N$ particles of a sequence forming one denoiser batch, in a single run; sharing inflates the absolute values, and the ratios between methods run under the same sharing are the quantity used.

\textbf{Numerical Verification of the Propositions.} On every steered sequence, the sum of log-potentials along the returned particle's lineage is compared with $r(x)/\alpha$; the maximum absolute discrepancy over all sequences, methods and schedules is below $10^{-13}$ in double precision. The identity of Proposition~\ref{prop:prefix} at $\beta = 0$ is checked in its discrete form: the sampler with $\beta = 0$ reproduces best-of-N importance sampling bitwise, including the final draw.

\textbf{Temperature.} On music the temperature is set from base samples before any steered sequence exists, by the rule $\alpha = \Delta / \ln 10^{3}$, where $\Delta$ is the mean over conditions of the interdecile range of the reward over that condition's base samples and their continuations; this makes the maximum weight ratio across a particle set of order $10^{3}$ at the terminal corner and is the same for every method. On dance the sweep $\alpha \in \{0.2, 0.05, 0.02\}$ is reported in full (Table~\ref{tab:rba_full}); the implied maximum weight ratio across 32 particles at these values is $3$, about $80$ and about $10^{4}$.

\section{Lookahead: Measurements and Twist Details}
\label{app:lookahead}

A twist has three sources (Figure~\ref{fig:lookahead}(c)). A prefix-evaluable reward is an exact intermediate potential by Proposition~\ref{prop:prefix}, with no estimation and no extra denoiser evaluations; it raises terminal CLAP by $0.020$ at $N = 32$ over terminal importance sampling, monotonically in $\beta$. Rollouts, $M$ continuations per particle averaging $\exp(S_k/\alpha)$, are an unbiased twist that establishes the headroom, $+0.058$ at $N = 32$ with sixteen continuations on repetition, but charged at their cost they lose to additional particles for $M \le 4$ (Appendix~\ref{app:rollout}). A learned predictor is the cheapest twist, but its rollout-derived targets carry a Monte Carlo error of about 2 nats against a within-set spread of about 3, so they hold no within-set ranking to learn; one trained directly on within-set weights reaches a within-set correlation of $0.08$ to $0.13$ (Appendix~\ref{app:twistvariants}).

\subsection{Estimating within-set predictability}
\label{app:decomp}

\begin{figure}[h]
\centering
\includegraphics[width=\textwidth]{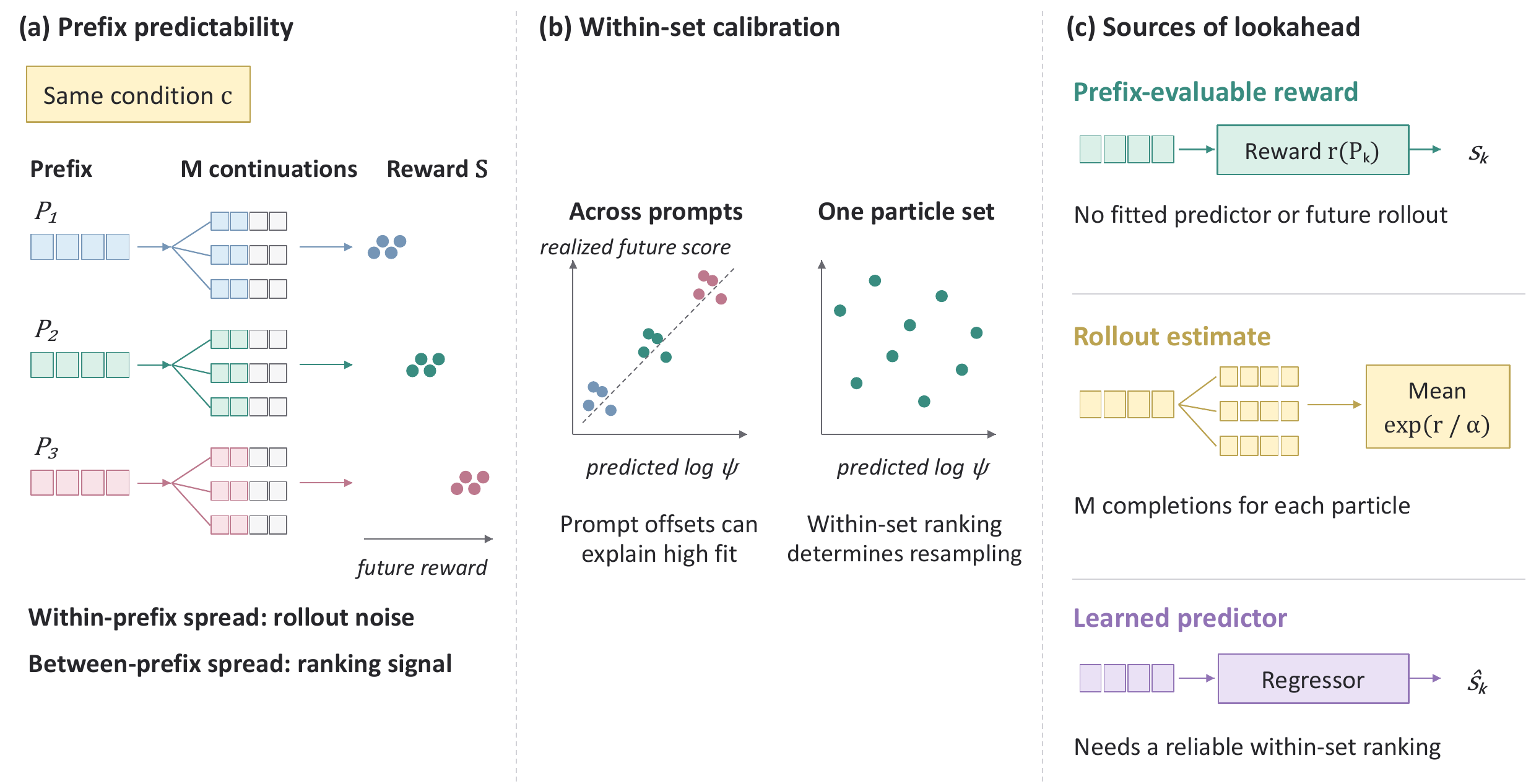}
\caption{\textbf{When lookahead pays (schematic).} (a) Within-set predictability: prefixes drawn under one condition are each continued $M$ times; the spread within a prefix is rollout noise, the spread between prefixes is the signal a twist can rank on (\Cref{eq:decomp}). (b) Twist accuracy is a within-set property: a regressor can fit realized future reward across prompts through prompt-level offsets while ranking the particles of one set at random (measured values in Table~\ref{tab:twist} and Appendices~\ref{app:withinset} and~\ref{app:twistvariants}). (c) Three sources of a content-axis twist: a prefix-evaluable reward (exact, free), rollouts (unbiased, charged), a learned predictor (cheap, validate within sets).}
\label{fig:lookahead}
\end{figure}

For a condition $c$ and chunk index $k$, prefixes $x^{(\le k)}_{1:P}$ are drawn from the base model and each is continued $M$ times to the end, giving future rewards $S_{p,m}$. With $\bar S_p$ the per-prefix mean and $v_p$ the per-prefix sample variance, the within-prefix term of \Cref{eq:decomp} is $\bar v = \frac{1}{P}\sum_p v_p$ and the between-prefix term is $\mathrm{Var}_p[\bar S_p] - \bar v / M$, the correction removing the Monte Carlo variance that the finite $M$ adds to every prefix mean. The within-set predictability is $\rho_k = (\mathrm{Var}_p[\bar S_p] - \bar v/M) / (\mathrm{Var}_p[\bar S_p] - \bar v/M + \bar v)$, clipped at zero, with a bootstrap standard error over prefixes. Pooling prefixes across conditions gives the between-condition version of the same quantity, which is what a cross-prompt regression fits; Table~\ref{tab:predictability} reports both.

\begin{table}[h]
\caption{Within-set predictability of future reward from base-model rollouts. Pooled values draw prefixes from many conditions; within-condition values hold the condition fixed. Beat alignment and repetition use the dance model; prompt adherence uses the music model. The repetition and prompt-adherence rows decompose the variance of $S_k$ (the conditional-mean diagnostic of \Cref{eq:decomp}); the beat-alignment row decomposes the variance of the exponential target $\exp(S_k/\alpha)$ itself and reports the ratio of between- to within-prefix variance, an upper bound on the share. The within-condition repetition share is bounded through the rollout-target spread of Appendix~\ref{app:twistvariants}.}
\label{tab:predictability}
\begin{center}
\small
\begin{tabular}{llccc}
\toprule
Reward & Prefix pool & $k=1$ & $k=2$ & $k=3$ \\
\midrule
Beat alignment (chunk-additive) & pooled, $M=8$ & $\le 0.13$ & $\le 0.13$ & $\le 0.14$ \\
Repetition (long range) & pooled, $M=16$ & $0.92$ & $0.94$ & $0.93$ \\
Prompt adherence (terminal) & within condition, $M=8$ & $0.47$ & $0.62$ & $0.66$ \\
\bottomrule
\end{tabular}
\end{center}
\end{table}

\subsection{Within-set accuracy of a twist}
\label{app:withinset}

For a particle set produced by the sampler at chunk $k$, the oracle twist $\log \psi^{\star}_k$ of every particle is computed from $M = 16$ continuations to the end. A candidate twist's within-set accuracy is the Pearson (and Spearman) correlation between its $\log \psi_k$ and $\log \psi^{\star}_k$ across the $N$ particles of the set, averaged over sets; its weight-level error is the standard deviation of $\log \psi_k - \log \psi^{\star}_k$ within a set after removing the set mean, since the set mean cancels in the normalized weights. The primary learned twist is a multilayer perceptron ($4{,}638 \to 512 \to 512 \to 1$, ReLU, inputs standardized) whose input is the last 30 frames of the prefix in the chunk's own root frame, the mean-pooled music features of each remaining chunk and a one-hot chunk index; it is trained on 4{,}000 prefixes (3{,}500 training, 500 validation) with rollout targets $\log \frac{1}{M}\sum_m \exp(S_{p,m}/\alpha)$ at $M = 8$ (Monte Carlo standard error about $0.5$ nats), by mean-squared error with Adam (learning rate $10^{-3}$, weight decay $10^{-4}$, batch 256, 300 epochs, the checkpoint with the lowest validation error kept), and evaluated on 500 fresh prefixes with fresh rollouts. Its residual column in Table~\ref{tab:twist} is the root-mean-square error against those targets across prefixes; within a set, after removing the set mean, the residual standard deviation is $6.0$ nats. Its residual on held-out prefixes decomposes into a between-sequence component (0.27 of its variance) and a within-sequence component (0.73), so most of its error is of the kind that does not cancel.

\subsection{Twist variants that rescale signal and noise together}
\label{app:twistvariants}

Tempering replaces $\log \psi_k$ by $\beta \log \psi_k$; it scales the twist's contribution to the log-weights, signal and error alike, by $\beta$, so the root-mean-square residual of $\beta\log\psi_k$ against the tempered target $\beta\log\psi^{\star}_k$ on fresh prefixes falls from $4.3$ nats at $\beta = 1$ to $2.1$, $1.1$ and $0.4$ at $\beta = 0.5$, $0.25$ and $0.1$, while the within-set correlation is unchanged and the within-set residual against the untempered oracle stays at $3.6$, $3.0$ and $3.0$ nats, the oracle's own within-set spread: tempering removes the twist's influence, not its error. The effect on the returned reward at $N = 32$ moves from $-0.053$ to $-0.017$, $+0.003$ and $-0.002$, the last two indistinguishable from $\psi \equiv 1$, and the fraction of early boundaries at which the sampler resamples falls from $1.00$ to $0.05$. An ensemble of five regressors agrees with itself to about 1 nat while remaining about 4 nats from the target, so the ensemble mean ($4.1$ nats) and a shrinkage by the ensemble variance fitted on validation prefixes (shrink factor $0.996$) reproduce the single regressor's effect ($-0.048$ and $-0.052$). A regressor trained directly on within-set targets, with a cross-entropy between the true and predicted normalized weights of each set plus a set-centred regression term, on 200 sets of 16 same-condition prefixes each continued 8 times, reaches a held-out within-set Spearman correlation of $0.08$ ($k=1$) and $0.13$ ($k=2$); its training targets carry a Monte Carlo error of about $2$ nats per particle against a within-set spread of $2.9$ nats, which limits what the regressors we trained could extract from prefix-level rollouts at this $M$.

\subsection{Rollout lookahead at matched cost}
\label{app:rollout}

A rollout twist with $M$ continuations to the end at chunks $k = 1$ and $2$ of $K = 4$ costs $(4 + 5M)/4$ times the denoiser evaluations of a plain run with the same $N$ (the continuations cover $3M + 2M$ chunks per particle on top of the four of the run itself), so it is compared against the $\psi \equiv 1$ sampler at $N' = N(4 + 5M)/4$ particles, with every continuation's evaluations counted. On the repetition reward the rollout twist trails at every pair: $-0.029$, $-0.058$ and $-0.071$ at $N = 4$ for $M = 1, 2, 4$, and $-0.009$, $-0.017$ and $-0.040$ at $N = 8$ (paired bootstrap intervals exclude zero except at $N = 8$, $M = 1$). The rollout estimate's Monte Carlo error at these $M$, measured by subsampling the stored continuations, is $5.5$, $3.7$ and $2.2$ nats against the $M = 8$ estimate, which is why the resampling it triggers at 67 to 98 percent of early boundaries removes diversity the extra particles would have kept. The oracle at $M = 16$, whose gain of $+0.058$ defines the headroom, spends $16 \times 100 \times (3 + 2) = 8{,}000$ denoiser evaluations per particle on its continuations, 20 times the $400$ of the sequence itself.

\subsection{Denoising-axis plug-in accuracy}
\label{app:plugin}

Along the denoising axis the plug-in prediction of the current chunk's own reward is evaluated against the realized reward of the same continuation at DDIM steps 45, 35, 25, 15 and 5 of 50 on 500 fresh prefixes. For the repetition reward its $R^2$ is $0.995$, $0.985$, $0.974$, $0.952$ and $0.858$; for beat alignment the plug-in explains $0.40$ of the variance at step 45 and none from step 35 on, its correlation with the realized reward falling to $0.11$ at step 5. This is the Jensen-gap ordering of Appendix~\ref{app:jensen}: chunk-local lookahead is accurate late in denoising and exact at the chunk's end, so within-chunk resampling events leave the chunk weight exact when the plug-in is free to score; what they change is the resampling times, whose finite-particle effect Section~\ref{sec:mainresult} measures.

\section{Experimental Details}
\label{app:experiments}

\subsection{Testbeds}
\label{app:testbeds}

\begin{figure}[h]
\centering
\includegraphics[width=\textwidth]{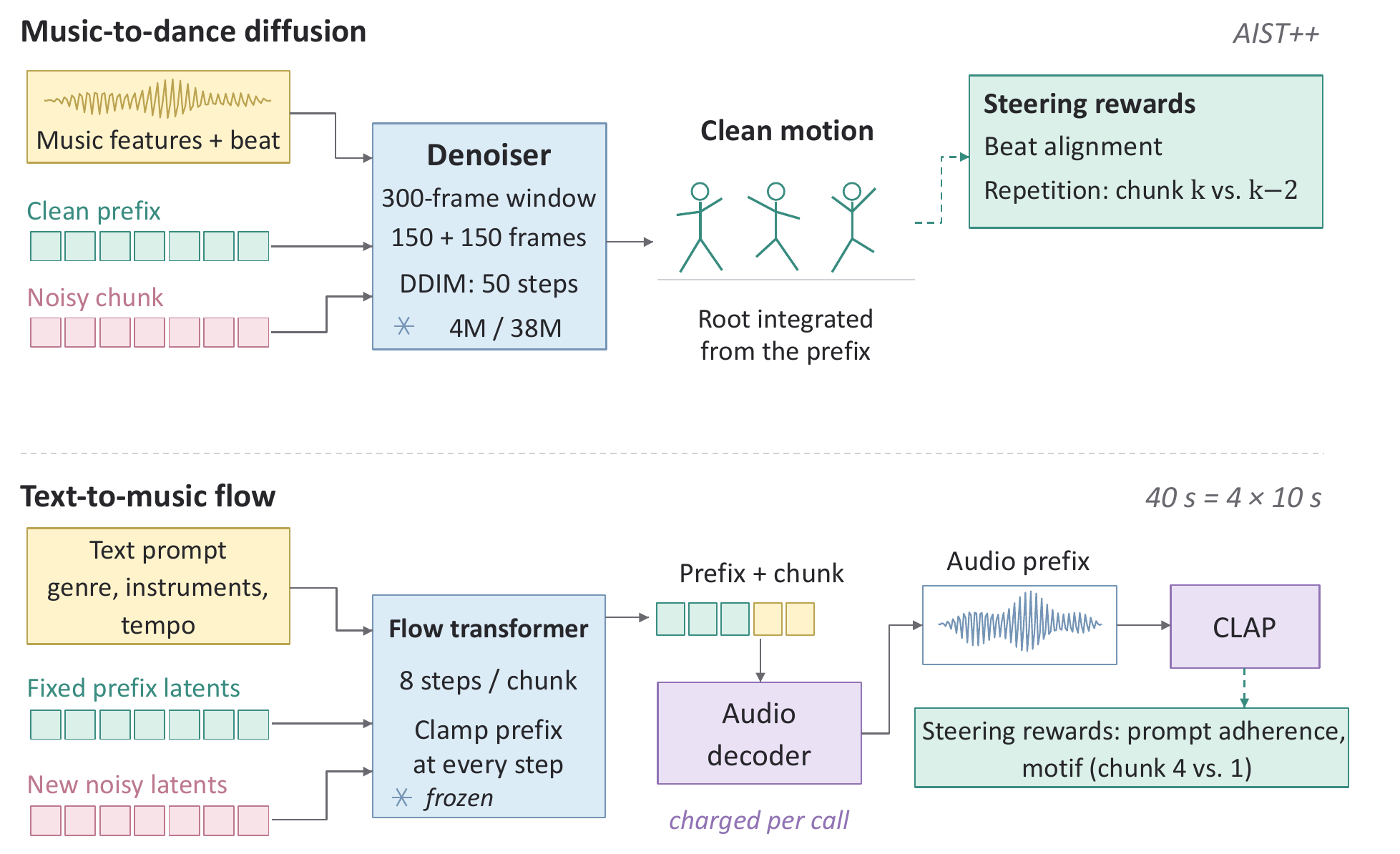}
\caption{\textbf{Two chunked generators.} Top: music-to-dance diffusion on AIST++. The denoiser sees a 300-frame window, the clean prefix chunk and the noisy chunk, and predicts the chunk's root as a displacement integrated from the prefix root, so the prefix is fixed and the seam continuous by construction; steering rewards are beat alignment (chunk-additive) and repetition of chunk $k$ against chunk $k-2$ (long range). Bottom: text-to-music flow, 40 s as four 10 s chunks. Each chunk is generated by inpainting past the clip's edge with the prefix latents replaced at every flow step; scoring a prefix requires one decode and one CLAP pass, which is charged; steering rewards are prompt adherence (terminal, evaluable on any prefix) and motif recurrence of the last chunk against the first.}
\label{fig:testbeds}
\end{figure}

\textbf{Chunked Dance Diffusion.} The base model is a transformer denoiser in the EDGE representation \citep{tseng2023edge} (151 channels: foot contacts, root translation and 6D joint rotations at 30 fps) trained on AIST++ \citep{li2021aistpp} with EDGE's per-frame losses. Chunking is obtained by training on pairs of consecutive 150-frame windows: the input is a 300-frame window whose first half is the clean prefix chunk and whose second half is the noised chunk, with a binary prefix-indicator channel, and the loss is computed on the second half, with the velocity and foot-sliding terms at the seam evaluated against the clean prefix frame. With probability $0.2$ the prefix is replaced by zeros to train the unconditional start of a sequence. The generated chunk predicts its root as a per-frame displacement in metres, normalized by its training standard deviation and integrated from the prefix root at sampling time, so root continuity at the seam holds by construction; the pose seam is checked by the ratio of the joint-position jump at the seam to the median within-chunk jump, $2.87$ for the 4M model and $2.27$ for the 38M model against $0.93$ for ground truth. The 4M model (latent 192, feed-forward 384, 4 layers, 4 heads, dropout 0.1, 4.06M parameters) is scheduled for 18{,}000 steps and the 38M model (latent 512, feed-forward 1024, 6 layers, 8 heads, 37.6M parameters) for 36{,}000, at batch 64 with Adan (learning rate $4 \times 10^{-4}$, weight decay 0.02, constant schedule), EDGE's cosine schedule with 1{,}000 diffusion steps and $x_0$ prediction, condition dropout 0.25, and EDGE's loss weights (simple 0.636, velocity 2.964, forward-kinematic 0.646, foot contact 10.942) applied to the chunk half of the window; the audio features are EDGE's 35-dimensional baseline features at 30 fps, whose beat channel the reward reads. Training uses 7{,}090 windows of 300 frames from the training sequences and 2{,}620 from held-out sequences for monitoring; each model adopts the checkpoint with the lowest held-out loss (step 12{,}000 and step 9{,}000). The full configuration files are in the supplementary material. Sampling is EDGE's DDIM schedule with 50 steps, $\eta = 1$ and guidance weight 2, so a chunk costs $2 \times 50$ denoiser rows per particle. Ninety-six music-and-motion sequences are held out at the sequence level; prompts are 20-second music segments from held-out sequences, two per music piece across all ten genres, 40 in total, with four seeds each.

\begin{figure}[h]
\centering
\includegraphics[width=\textwidth]{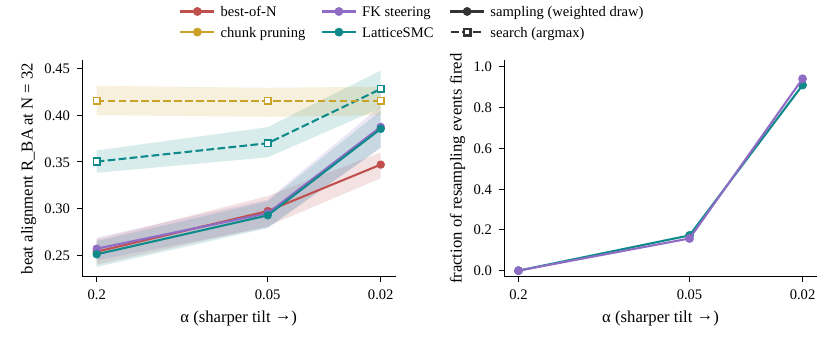}
\caption{\textbf{Temperature and effective sample size on beat alignment.} Left: reward at $N = 32$ for the proper methods and chunk pruning at $\alpha \in \{0.2, 0.05, 0.02\}$; the boundary and dense schedules of \method{} coincide at every temperature, and both approach pruning as the tilt sharpens. Right: the fraction of chunk boundaries at which the ESS falls below $N/2$ and resampling fires; at $\alpha = 0.2$ the maximum weight ratio across 32 particles is 3 and no method resamples, at $\alpha = 0.02$ it is of order $10^{4}$ and resampling fires at 91 to 98 percent of boundaries.}
\label{fig:temperature}
\end{figure}

\textbf{Chunked Text-to-Music Flow.} The base model is Stable Audio 3 Medium \citep{stabilityai2026sa3}, a 44.1 kHz stereo latent flow-matching transformer \citep{lipman2023flow} with an autoencoder (Hugging Face \texttt{stabilityai/stable-audio-3-medium}, revision \texttt{27b5a21}), run at its shipped settings (8 ping-pong steps, guidance 1.0, so one denoiser row per particle per step; denoiser in bf16, decoder in fp16 as shipped); the decoder is the checkpoint's own autoencoder, and the CLAP scorer is the LAION HTSAT-base music checkpoint \texttt{music\_audioset\_epoch\_15\_esc\_90.14} \citep{wu2023clap}. The chunked extension is this study's adaptation of inpainting-based extension: each new 10-second chunk is generated with the previous latents present, and at every flow step the latent region of the prefix is overwritten with the fixed prefix latents noised to that step, so the particle state is bit-exact across chunks. The decoder is stochastic and is given a seed fixed per (prompt, seed) pair, which makes waveforms reproducible across processes. The seam spectral-flux ratio of the raw decode is $1.26$ (a within-chunk control gives $1.21$), tempo adherence of base samples to the prompted tempo is $0.95$ within 5 percent, and CLAP prompt retrieval over 10 genre-diverse prompts is $0.65$ at rank 1 (chance $0.1$). Prompts are 40 templated descriptions (8 genres, instrumentation, an explicit tempo, no structural words), with two seeds each.

\subsection{Rewards and held-out metrics}
\label{app:rewards}

\textbf{Beat Alignment.} Kinematic beats are local minima of the smoothed mean joint speed of the generated motion \citep{siyao2022bailando}; the reward is the mean over music beats of $\exp(-d^2 / 2\sigma^2)$ with $d$ the distance in frames to the nearest kinematic beat and $\sigma = 3$ frames. Its chunk-additive form assigns each music beat to its chunk and evaluates it against the kinematic beats of the prefix through that chunk; the mean absolute gap between the additive and full forms is $0.01$ on every method.

\textbf{Repetition.} $r_k = \exp(-d_k / \tau)$ for $k \ge 3$ with $d_k$ the mean root-relative joint-position distance between chunk $k$ and chunk $k-2$ and $\tau = 0.171$ m, the median of $d$ over base samples; $r_1 = r_2 = 0$. The additive and full forms coincide exactly. Real dancers score $0.412$ on this reward and the base model $0.718$.

\textbf{Prompt Adherence.} The cosine similarity between the CLAP audio embedding of the clip and the CLAP text embedding of its prompt, with the LAION music checkpoint \citep{wu2023clap}; a 40-second clip is embedded as the renormalized mean of its four 10-second window embeddings. The prefix score $s_k$ is the same quantity on the decoded prefix through chunk $k$; the additive counterpart, the mean of the window scores, differs from the terminal value by under 2 percent.

\textbf{Motif Recurrence.} The cosine similarity between the CLAP audio embeddings of the fourth and first 10-second windows of the returned clip; the prefix score is the similarity of the third window to the first, available at the last boundary. On base samples $99.4$ percent of clips have a final window closer to their own opening than to another prompt's, so the reward moves within the model's own notion of a returning motif.

\textbf{Held-out Metrics.} On dance: EDGE's physical foot-contact score (PFC), a realism statistic (the 1-Wasserstein distance between the per-frame joint-acceleration distribution of a sample and that of held-out real motion), root travel per sequence, and pose diversity (mean pairwise root-relative joint distance across the seeds of a condition). On music: tempo adherence (fraction of chunks whose tracked tempo is within 5 percent of the prompted tempo, with a neural beat tracker \citep{foscarin2024beatthis}), the fraction of energy above 8 kHz, the four audiobox-aesthetics axes \citep{tjandra2025audiobox}, the seam spectral-flux ratio, and embedding diversity (mean pairwise CLAP cosine distance across seeds). None is used by any method.

\subsection{Cost accounting and compute}
\label{app:expimpl}

Denoiser rows, decoder calls and reward evaluations are counted per sequence and audited against their closed forms; per sequence at $N = 32$ the music methods make 32 (best-of-N), 128 (pruning and \method{} on prompt adherence, where every boundary is scored; 64 on motif recurrence, whose prefix score exists only at the last boundary) and 512 (dense schedule on prompt adherence) decoder and CLAP calls. Wall-clock is measured with three processes sharing a GPU, so the ratios between methods are the informative quantity. The main study, including the scale replicate and the lookahead measurements of Appendix~\ref{app:lookahead}, used about 250 GPU-hours on two 48 GB GPUs; the pruning sweep, the horizon and temperature runs and the music qualification of Appendices~\ref{app:prunesweep} and~\ref{app:horizon} used about 110 more, measured as summed job walls on GPUs shared with other jobs.

\subsection{Temperature, cost axes and crossover}
\label{app:temperature}

Figures~\ref{fig:temperature}, \ref{fig:cost} and~\ref{fig:crossover} give the temperature sweep on beat alignment, the three cost axes on the music model and the crossover against pruning by particle count that Sections~\ref{sec:returnrule} and~\ref{sec:costaxes} summarize.

\begin{figure}[h]
\centering
\includegraphics[width=\textwidth]{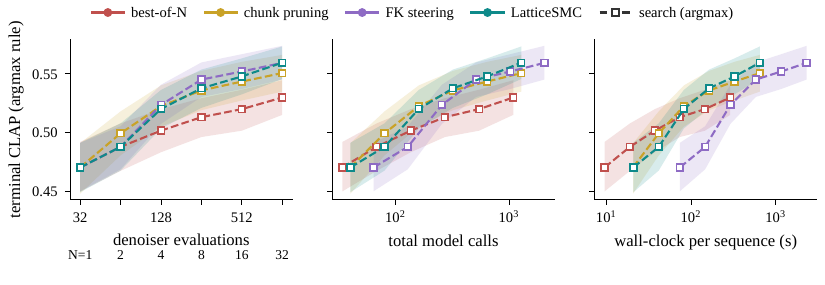}
\caption{\textbf{Three cost axes on the music model (prompt adherence).} The same runs plotted against denoiser evaluations, total model calls (denoiser, decoder and CLAP) and measured wall-clock per sequence. The two schedules of \method{} carry identical weights (the plug-in CLAP differences telescope to the prefix-score increments of Proposition~\ref{prop:prefix}) and coincide on the first axis; on a latent generator each within-chunk event costs a decode and a reward pass, so the dense schedule moves right by $3.6\times$ against the boundary schedule and its curve flattens from $N = 8$, while the boundary schedule keeps rising. Best-of-N is cheapest per particle and lowest in reward at every budget.}
\label{fig:cost}
\end{figure}

\begin{figure}[h]
\centering
\includegraphics[width=\textwidth]{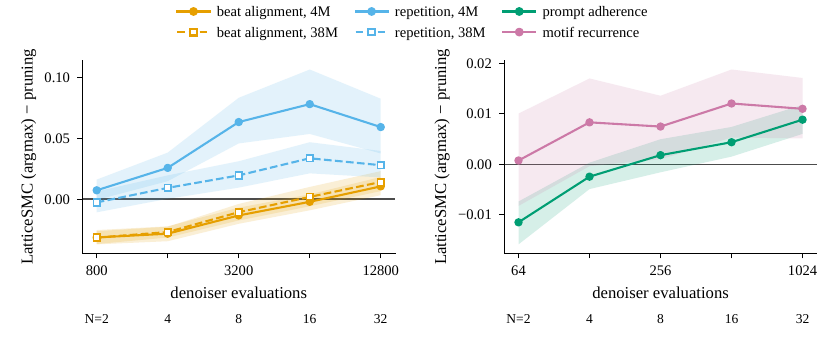}
\caption{\textbf{Proper resampling overtakes greedy pruning as the budget grows.} Paired difference between \method{} with argmax return and chunk pruning as a function of particle count, with 95\% bootstrap intervals over conditions, for the four rewards and, on the dance rewards, both model sizes. Pruning's early commitment leads at the smallest budgets on beat alignment and prompt adherence; ESS-gated resampling crosses it between $N = 16$ and $32$ on beat alignment, at $N = 8$ on prompt adherence and at $N = 4$ on motif recurrence, and leads from $N = 4$ on the long-range repetition reward, where pruning replicates blind for two chunks. The 38M model reproduces the 4M crossovers.}
\label{fig:crossover}
\end{figure}

\subsection{Pruning strength}
\label{app:prunesweep}

Chunk pruning keeps the top $M$ particles by prefix reward at every boundary and expands each to $N/M$ continuations; Table~\ref{tab:main} runs it at $M = N/4$, isolating the beam expansion and top-$k$ retention of Stream-T1 \citep{streamt12026} from its noise-initialization and memory components, which the common generator kernel replaces. Table~\ref{tab:prunesweep} sweeps $M \in \{1, N/4, N/2\}$ at $N = 32$ on all four rewards with the same conditions, seeds, temperatures and denoiser budget (12{,}800 rows per dance sequence; 1{,}024 rows and 128 (prompt adherence) or 64 (motif recurrence) decoder and CLAP calls per music sequence for every $M$). The held-out quantities of every row lie inside the base intervals (dance: PFC $[1.25, 2.72]$, $W_1$ $[0.0066, 0.0094]$; music: base tempo $0.903$, 8 kHz fraction $0.0069$, seam $1.27$, diversity $0.061$); the motif rows' held-out terminal CLAP is $0.465$ to $0.468$ against $0.470$ for the base. On repetition every prefix reward is zero at the first two boundaries, so top-$M$ selection there is a tie broken by index and $M = 1$ commits to an arbitrary lineage before any reward has arrived; the ESS gate leaves the set intact at those boundaries, which is the mechanism behind Section~\ref{sec:costaxes}.

\begin{table}[h]
\caption{Chunk pruning at $N = 32$ as a function of its strength $M$ (argmax by construction; $M = 1$ is the $\alpha \to 0$ limit of the boundary schedule of \method{} with argmax return when the maximum is unique), with 95\% bootstrap intervals over conditions, held-out quantities, and paired differences against the $M = N/4$ row and against the boundary schedule of \method{} with argmax return; an asterisk marks an interval excluding zero. Dance: 4M model, $\alpha = 0.02$, held-out PFC $\times 10^{4}$ and $W_1$. Music: held-out tempo adherence, 8 kHz energy fraction and seed diversity.}
\label{tab:prunesweep}
\begin{center}
\footnotesize
\setlength{\tabcolsep}{4pt}
\resizebox{\textwidth}{!}{\begin{tabular}{l c l l cc}
\toprule
Reward & $M$ & Reward [95\% CI] & Held out & $-$ ($M = N/4$) & $-$ boundary argmax \\
\midrule
Beat alignment & 1 & $\mathbf{0.473}\ [0.453, 0.491]$ & PFC 1.73, $W_1$ 0.0078 & $+0.058^{*}$ & $+0.045^{*}$ \\
 & 8 & $0.415\ [0.399, 0.431]$ & PFC 2.05, $W_1$ 0.0082 & & $-0.013^{*}$ \\
 & 16 & $0.386\ [0.372, 0.400]$ & PFC 1.80, $W_1$ 0.0076 & $-0.029^{*}$ & $-0.042^{*}$ \\
\midrule
Repetition & 1 & $0.850\ [0.755, 0.941]$ & PFC 1.62, $W_1$ 0.0081 & $-0.037^{*}$ & $-0.096^{*}$ \\
 & 8 & $0.887\ [0.796, 0.976]$ & PFC 1.36, $W_1$ 0.0077 & & $-0.059^{*}$ \\
 & 16 & $0.916\ [0.826, 1.013]$ & PFC 1.43, $W_1$ 0.0081 & $+0.028^{*}$ & $-0.031^{*}$ \\
\midrule
Prompt adherence & 1 & $\mathbf{0.564}\ [0.551, 0.577]$ & tempo 0.949, 8 kHz 0.0064, div.\ 0.049 & $+0.013^{*}$ & $+0.004^{*}$ \\
 & 8 & $0.551\ [0.535, 0.566]$ & tempo 0.894, 8 kHz 0.0067, div.\ 0.052 & & $-0.009^{*}$ \\
 & 16 & $0.543\ [0.526, 0.558]$ & tempo 0.934, 8 kHz 0.0061, div.\ 0.055 & $-0.008^{*}$ & $-0.016^{*}$ \\
\midrule
Motif recurrence & 1 & $0.909\ [0.899, 0.919]$ & tempo 0.918, 8 kHz 0.0073, div.\ 0.069 & $-0.009^{*}$ & $-0.020^{*}$ \\
 & 8 & $0.918\ [0.908, 0.927]$ & tempo 0.925, 8 kHz 0.0073, div.\ 0.063 & & $-0.011^{*}$ \\
 & 16 & $0.924\ [0.916, 0.931]$ & tempo 0.922, 8 kHz 0.0077, div.\ 0.051 & $+0.005^{*}$ & $-0.006^{*}$ \\
\bottomrule
\end{tabular}}
\end{center}
\end{table}

\subsection{Horizon and per-chunk tilt}
\label{app:horizon}

\textbf{Dance.} The paper's 40 prompts are 20-second segments of 40 held-out AIST++ sequences whose lengths are 6 to 9 chunks of 5 s. $K = 6$ (30 s) therefore exists for every condition, with each segment containing the paper's 20 s segment, whereas a 40 s segment exists for only 16 of the 40 source sequences, so $K = 8$ is run on those 16 conditions (8 of the 20 music pieces) with four seeds and its intervals are over 16 conditions, not pooled with the 40. Every method runs at $\alpha = 0.02$, the paper's value at $K = 4$, and the two \method{} schedules are additionally run at $\alpha = 0.01$ on beat alignment, a second point fixed before the runs; the denoiser budget is $100KN$ rows per sequence, audited. Tables~\ref{tab:k6} and~\ref{tab:k8} give the $N = 8$ and $N = 32$ rows with the fraction of the $K-1$ boundaries at which $\mathrm{ESS} < N/2$ triggered resampling, the range of the mean $\mathrm{ESS}/N$ before the boundaries, and the held-out quantities against the base model's bootstrap interval on the same conditions. On the 16 conditions of $K = 8$, the paper's $K = 4$ grids give at $N = 32$: beat alignment, boundary argmax $0.464$ against pruning $0.431$ ($+0.033\ [+0.023, +0.043]$) and dense argmax $0.483$; repetition (two pairs at $K = 4$), boundary argmax $0.888$ against pruning $0.803$ ($+0.086\ [+0.040, +0.135]$). Repetition at $K = 6$ and $8$ is scored as the sum the methods steered on, over four and six pairs; per pair the boundary schedule's lead over pruning is $0.030$, $0.025$ and $0.020$ at $K = 4$, $6$ and $8$.

The reading in Section~\ref{sec:costaxes} rests on the resampling rate and on how each reward is normalized. Beat alignment is a mean over all beats of the sequence (range $[0,1]$ at every $K$), so its chunk term $r_k$ is chunk $k$'s share of that mean and the chunk potential $\exp(r_k/\alpha)$ weakens as $1/K$ while the terminal potential $\exp(r/\alpha)$ keeps its scale; repetition is a sum over the $K-2$ chunk pairs, so its per-pair potential keeps its scale at every $K$ while the terminal range grows as $K-2$, which is why its ordering holds at every horizon without changing $\alpha$. On beat alignment at a fixed $\alpha$ a longer chain therefore carries a weaker per-boundary tilt, the mean $\mathrm{ESS}/N$ before a boundary rises from the $0.18$--$0.25$ of the paper's $K = 4$ regime to $0.41$--$0.50$ at $K = 6$ and $0.46$--$0.58$ at $K = 8$, and the ESS gate fires less often, while top-$M$ pruning commits at every boundary independently of $\alpha$. At $\alpha = 0.01$ the mean $\mathrm{ESS}/N$ returns to $0.18$--$0.25$, the gate fires at 99 and 98 percent of boundaries, and the boundary schedule leads pruning under the argmax rule at both horizons, while under the draw rule both schedules reach pruning's level ($+0.008$ and $+0.006$ for the boundary schedule, intervals containing zero) after trailing it by $0.07$ to $0.11$ at $\alpha = 0.02$. A rule that keeps the per-boundary tilt fixed, $\alpha \propto 1/K$ for a chunk-additive reward, is the principled default; it was not applied in these runs, which keep $\alpha = 0.02$ at every $K$ with $0.01$ as the pre-registered second point.

\begin{table}[h]
\caption{Horizon $K = 6$ (30 s, 40 conditions, 4M dance model): full reward of the returned sequence with 95\% bootstrap intervals, fraction of boundaries at which resampling fired, range of the mean $\mathrm{ESS}/N$ before the boundaries, and held-out quantities (PFC $\times 10^{4}$, $W_1$, root-relative pose diversity). Base intervals on these conditions: PFC $[1.05, 2.51]$, $W_1$ $[0.0065, 0.0094]$. Paired differences at $N = 32$, beat alignment: boundary argmax $-$ pruning $-0.017\ [-0.024, -0.010]$; dense argmax $-$ pruning $-0.010\ [-0.016, -0.003]$; boundary argmax, $\alpha = 0.01$, $-$ pruning $+0.026\ [+0.019, +0.033]$; dense argmax, $\alpha = 0.01$, $-$ pruning $+0.031\ [+0.021, +0.042]$; boundary argmax, $\alpha = 0.01$, $-$ best-of-N argmax $+0.103\ [+0.092, +0.113]$. Repetition (sum over the four chunk pairs; rows give the mean of the returned sequence's sum): boundary argmax $-$ pruning $+0.100\ [+0.070, +0.135]$; boundary argmax $-$ best-of-N argmax $+0.135\ [+0.099, +0.164]$; dense argmax $-$ boundary argmax $-0.001\ [-0.046, +0.041]$.}
\label{tab:k6}
\begin{center}
\footnotesize
\setlength{\tabcolsep}{4pt}
\begin{tabular}{l c l c c c c c}
\toprule
Method & $N$ & Reward [95\% CI] & fired & $\mathrm{ESS}/N$ & PFC & $W_1$ & diversity \\
\midrule
\multicolumn{8}{l}{\emph{Beat alignment}, $\alpha = 0.02$} \\
Base & 1 & $0.236\ [0.228, 0.243]$ &  &  & 1.70 & 0.0078 & 0.178 \\
Best-of-N (draw) & 8 & $0.292\ [0.283, 0.301]$ &  &  & 1.58 & 0.0075 & 0.180 \\
Best-of-N (argmax) & 8 & $0.296\ [0.288, 0.304]$ &  &  & 1.75 & 0.0079 & 0.178 \\
Chunk pruning & 8 & $0.376\ [0.365, 0.387]$ & 1.00 & 0.25 & 1.90 & 0.0079 & 0.181 \\
\method{}, boundary (draw) & 8 & $0.310\ [0.298, 0.323]$ & 0.35 & 0.53--0.62 & 1.69 & 0.0079 & 0.185 \\
\method{}, boundary (argmax) & 8 & $0.331\ [0.319, 0.343]$ & 0.35 & 0.53--0.62 & 1.75 & 0.0078 & 0.177 \\
\method{}, dense (draw) & 8 & $0.321\ [0.308, 0.334]$ &  &  & 1.91 & 0.0077 & 0.185 \\
\method{}, dense (argmax) & 8 & $0.347\ [0.333, 0.360]$ &  &  & 2.03 & 0.0079 & 0.182 \\
Best-of-N (draw) & 32 & $0.323\ [0.311, 0.335]$ &  &  & 1.82 & 0.0078 & 0.179 \\
Best-of-N (argmax) & 32 & $0.333\ [0.322, 0.344]$ &  &  & 1.69 & 0.0078 & 0.179 \\
Chunk pruning & 32 & $0.410\ [0.396, 0.424]$ & 1.00 & 0.25 & 2.13 & 0.0086 & 0.176 \\
\method{}, boundary (draw) & 32 & $0.341\ [0.324, 0.357]$ & 0.64 & 0.41--0.50 & 1.74 & 0.0076 & 0.179 \\
\method{}, boundary (argmax) & 32 & $0.393\ [0.373, 0.412]$ & 0.64 & 0.41--0.50 & 1.91 & 0.0079 & 0.180 \\
\method{}, dense (draw) & 32 & $0.344\ [0.328, 0.359]$ &  &  & 1.85 & 0.0084 & 0.177 \\
\method{}, dense (argmax) & 32 & $0.400\ [0.382, 0.417]$ &  &  & 1.85 & 0.0085 & 0.176 \\
\multicolumn{8}{l}{\emph{Beat alignment}, $\alpha = 0.01$} \\
\method{}, boundary (draw) & 32 & $0.417\ [0.398, 0.437]$ & 0.99 & 0.18--0.21 & 2.00 & 0.0086 & 0.180 \\
\method{}, boundary (argmax) & 32 & $0.436\ [0.417, 0.454]$ & 0.99 & 0.18--0.21 & 1.79 & 0.0082 & 0.181 \\
\method{}, dense (draw) & 32 & $0.420\ [0.399, 0.443]$ &  &  & 2.06 & 0.0083 & 0.174 \\
\method{}, dense (argmax) & 32 & $0.441\ [0.420, 0.462]$ &  &  & 2.11 & 0.0087 & 0.174 \\
\midrule
\multicolumn{8}{l}{\emph{Repetition}, $\alpha = 0.02$ (sum over the $K-2$ chunk pairs)} \\
Base & 1 & $1.489$ &  &  & 1.70 & 0.0078 & 0.178 \\
Best-of-N (draw) & 32 & $1.793$ &  &  & 1.29 & 0.0075 & 0.142 \\
Best-of-N (argmax) & 32 & $1.795$ &  &  & 1.31 & 0.0074 & 0.143 \\
Chunk pruning & 32 & $1.830$ & 1.00 & 0.25 & 1.32 & 0.0079 & 0.160 \\
\method{}, boundary (draw) & 32 & $1.909$ & 0.59 & 0.14--1.00 & 1.13 & 0.0077 & 0.141 \\
\method{}, boundary (argmax) & 32 & $1.930$ & 0.59 & 0.14--1.00 & 1.11 & 0.0077 & 0.140 \\
\method{}, dense (draw) & 32 & $1.878$ &  &  & 1.70 & 0.0086 & 0.161 \\
\method{}, dense (argmax) & 32 & $1.930$ &  &  & 1.79 & 0.0088 & 0.159 \\
\bottomrule
\end{tabular}
\end{center}
\end{table}

\begin{table}[h]
\caption{Horizon $K = 8$ (40 s, the 16 conditions with an 8-chunk recording, 4M dance model), as Table~\ref{tab:k6}. Base intervals on these conditions: PFC $[0.76, 2.66]$, $W_1$ $[0.0046, 0.0092]$. Paired differences at $N = 32$, beat alignment: boundary argmax $-$ pruning $-0.052\ [-0.057, -0.046]$; dense argmax $-$ pruning $-0.038\ [-0.045, -0.029]$; boundary argmax, $\alpha = 0.01$, $-$ pruning $+0.021\ [+0.012, +0.031]$; dense argmax, $\alpha = 0.01$, $-$ pruning $+0.016\ [-0.002, +0.035]$; boundary argmax, $\alpha = 0.01$, $-$ best-of-N argmax $+0.120\ [+0.108, +0.134]$. Repetition (sum over the six chunk pairs; rows give the mean of the returned sequence's sum): boundary argmax $-$ pruning $+0.121\ [+0.060, +0.190]$; boundary argmax $-$ best-of-N argmax $+0.133\ [-0.118, +0.317]$; dense argmax $-$ boundary argmax $+0.035\ [-0.053, +0.106]$.}
\label{tab:k8}
\begin{center}
\footnotesize
\setlength{\tabcolsep}{4pt}
\begin{tabular}{l c l c c c c c}
\toprule
Method & $N$ & Reward [95\% CI] & fired & $\mathrm{ESS}/N$ & PFC & $W_1$ & diversity \\
\midrule
\multicolumn{8}{l}{\emph{Beat alignment}, $\alpha = 0.02$} \\
Base & 1 & $0.231\ [0.216, 0.245]$ &  &  & 1.62 & 0.0067 & 0.191 \\
Best-of-N (draw) & 8 & $0.290\ [0.272, 0.305]$ &  &  & 2.27 & 0.0088 & 0.214 \\
Best-of-N (argmax) & 8 & $0.296\ [0.277, 0.312]$ &  &  & 1.68 & 0.0070 & 0.197 \\
Chunk pruning & 8 & $0.392\ [0.370, 0.412]$ & 1.00 & 0.25 & 2.14 & 0.0084 & 0.213 \\
\method{}, boundary (draw) & 8 & $0.316\ [0.291, 0.336]$ & 0.33 & 0.50--0.65 & 2.54 & 0.0076 & 0.223 \\
\method{}, boundary (argmax) & 8 & $0.342\ [0.318, 0.360]$ & 0.33 & 0.50--0.65 & 2.06 & 0.0072 & 0.200 \\
\method{}, dense (draw) & 8 & $0.310\ [0.290, 0.330]$ &  &  & 1.63 & 0.0068 & 0.215 \\
\method{}, dense (argmax) & 8 & $0.343\ [0.321, 0.362]$ &  &  & 2.09 & 0.0078 & 0.216 \\
Best-of-N (draw) & 32 & $0.323\ [0.301, 0.342]$ &  &  & 2.32 & 0.0079 & 0.212 \\
Best-of-N (argmax) & 32 & $0.335\ [0.315, 0.351]$ &  &  & 2.23 & 0.0082 & 0.211 \\
Chunk pruning & 32 & $0.434\ [0.407, 0.455]$ & 1.00 & 0.25 & 3.02 & 0.0102 & 0.203 \\
\method{}, boundary (draw) & 32 & $0.328\ [0.302, 0.351]$ & 0.50 & 0.46--0.58 & 1.95 & 0.0077 & 0.201 \\
\method{}, boundary (argmax) & 32 & $0.382\ [0.354, 0.405]$ & 0.50 & 0.46--0.58 & 3.14 & 0.0091 & 0.210 \\
\method{}, dense (draw) & 32 & $0.332\ [0.309, 0.350]$ &  &  & 2.33 & 0.0093 & 0.202 \\
\method{}, dense (argmax) & 32 & $0.396\ [0.371, 0.417]$ &  &  & 2.92 & 0.0095 & 0.204 \\
\multicolumn{8}{l}{\emph{Beat alignment}, $\alpha = 0.01$} \\
\method{}, boundary (draw) & 32 & $0.440\ [0.406, 0.466]$ & 0.98 & 0.20--0.25 & 1.95 & 0.0065 & 0.202 \\
\method{}, boundary (argmax) & 32 & $0.455\ [0.429, 0.478]$ & 0.98 & 0.20--0.25 & 1.82 & 0.0069 & 0.203 \\
\method{}, dense (draw) & 32 & $0.427\ [0.388, 0.458]$ &  &  & 2.16 & 0.0088 & 0.202 \\
\method{}, dense (argmax) & 32 & $0.450\ [0.410, 0.485]$ &  &  & 2.44 & 0.0091 & 0.202 \\
\midrule
\multicolumn{8}{l}{\emph{Repetition}, $\alpha = 0.02$ (sum over the $K-2$ chunk pairs)} \\
Base & 1 & $2.077$ &  &  & 1.62 & 0.0067 & 0.191 \\
Best-of-N (draw) & 32 & $2.544$ &  &  & 1.80 & 0.0086 & 0.175 \\
Best-of-N (argmax) & 32 & $2.544$ &  &  & 1.80 & 0.0086 & 0.175 \\
Chunk pruning & 32 & $2.555$ & 1.00 & 0.25 & 1.36 & 0.0079 & 0.171 \\
\method{}, boundary (draw) & 32 & $2.650$ & 0.70 & 0.13--1.00 & 1.28 & 0.0080 & 0.158 \\
\method{}, boundary (argmax) & 32 & $2.676$ & 0.70 & 0.13--1.00 & 1.24 & 0.0080 & 0.158 \\
\method{}, dense (draw) & 32 & $2.652$ &  &  & 2.57 & 0.0097 & 0.172 \\
\method{}, dense (argmax) & 32 & $2.711$ &  &  & 2.75 & 0.0096 & 0.171 \\
\bottomrule
\end{tabular}
\end{center}
\end{table}

\textbf{Music.} The music generator's maximum clip length allows $K = 8$ (80 s) only with a rolling 40 s context, in which chunk $k \ge 5$ is generated with the previous three chunks as clamped prefix latents and the prefix score is evaluated on the decoded rolling window. Before any steered run, 20 base clips at $K = 8$ (10 prompts, 2 seeds) were compared with the paper's $K = 4$ base clips of the same prompts and seeds under a gate fixed in advance: a relative change of at most 10 percent in seam continuity and in tempo adherence. The seams passed (spectral-flux ratio $1.237$ against $1.261$), tempo adherence did not (fraction of clips within 5 percent of the prompted tempo $0.85$ against $0.95$, $-10.5$ percent; per-chunk tempo reward $-12.9$ percent), so the rolling-context generator is not the same base model over 80 s and the $K = 8$ music grid was not run. The horizon results therefore rest on dance alone.

\clearpage
\subsection{Full results by budget, scale and held-out quantities}
\label{app:fulltables}
\begin{table}[h]
\caption{Held-out quantities at $N = 32$, never used by any method, as ranges over all steered methods against the base; dance values for beat-alignment steering, music values for prompt-adherence steering. PFC $\times 10^{4}$ (lower is better; base 95\% interval $[1.25, 2.72]$); $W_1$ realism statistic (base interval $[0.0066, 0.0094]$); PQ is the production-quality aesthetics axis; diversity is across seeds.}
\label{tab:heldout}
\begin{center}
\footnotesize
\setlength{\tabcolsep}{4pt}
\resizebox{\textwidth}{!}{\begin{tabular}{l ccc ccccc}
\toprule
& \multicolumn{3}{c}{Dance} & \multicolumn{5}{c}{Music} \\
& PFC & $W_1$ & diversity & tempo & 8 kHz & PQ & seam & diversity \\
\midrule
Base & 1.97 & 0.0080 & 0.183 & 0.903 & 0.0069 & 8.17 & 1.27 & 0.061 \\
Steered, all methods & 1.84--2.24 & 0.0077--0.0083 & 0.174--0.184 & 0.894--0.944 & 0.0061--0.0067 & 8.12--8.16 & 1.25--1.34 & 0.048--0.060 \\
\method{}, argmax & 2.24 & 0.0081 & 0.183 & 0.924 & 0.0062 & 8.13 & 1.34 & 0.050 \\
\bottomrule
\end{tabular}}
\end{center}
\end{table}

\begin{table}[h]
\caption{Paired 95\% bootstrap intervals over the 40 conditions for the differences reported in Table~\ref{tab:main} and Section~\ref{sec:returnrule} at $N = 32$ (4M dance model unless stated). A dash marks a pair for which no paired interval is stored.}
\label{tab:intervals}
\begin{center}
\footnotesize
\resizebox{\textwidth}{!}{\begin{tabular}{l cccc}
\toprule
Difference & Beat alignment & Repetition & Prompt adherence & Motif recurrence \\
\midrule
boundary $-$ best-of-N, draw & $+0.038\ [+0.028, +0.048]$ & $+0.028\ [+0.019, +0.038]$ & $+0.020\ [+0.016, +0.024]$ & $-0.002\ [-0.009, +0.006]$ \\
boundary $-$ chunk pruning, argmax & $+0.013\ [+0.005, +0.021]$ & $+0.059\ [+0.039, +0.083]$ & $+0.009\ [+0.006, +0.012]$ & $+0.011\ [+0.005, +0.017]$ \\
boundary $-$ chunk pruning, draw & -- & $+0.036\ [+0.014, +0.060]$ & -- & -- \\
dense $-$ chunk pruning, argmax & $+0.026\ [+0.015, +0.039]$ & $+0.075\ [+0.052, +0.101]$ & -- & -- \\
dense $-$ boundary, argmax & $+0.013\ [+0.003, +0.025]$ & $+0.016\ [-0.008, +0.034]$ & $-0.000\ [-0.004, +0.005]$ & $+0.003\ [-0.003, +0.008]$ \\
dense $-$ boundary, draw & -- & -- & $-0.002\ [-0.007, +0.003]$ & $+0.006\ [-0.003, +0.015]$ \\
dense $-$ chunk pruning, argmax, 38M & $+0.032\ [+0.018, +0.047]$ & $+0.078\ [+0.062, +0.096]$ & -- & -- \\
dense $-$ boundary, argmax, 38M & $+0.017\ [+0.005, +0.030]$ & $+0.051\ [+0.041, +0.062]$ & -- & -- \\
argmax $-$ draw, boundary & $+0.042\ [+0.036, +0.049]$ & $+0.023\ [+0.019, +0.027]$ & $+0.012\ [+0.010, +0.015]$ & $+0.026\ [+0.022, +0.032]$ \\
\bottomrule
\end{tabular}}
\end{center}
\end{table}

\begin{table}[h]
\caption{Beat alignment (4M dance model), full reward of the returned sequence by particle count, draw rule unless marked; $\alpha = 0.02$. Argmax rows are regenerated at $N = 8$ and $32$ (dense schedule: $N = 32$); a dash marks a budget at which the argmax sequences were not regenerated, or the base model's single sample.}
\label{tab:rba_full}
\begin{center}
\small
\begin{tabular}{l cccccc}
\toprule
$N$ & 1 & 2 & 4 & 8 & 16 & 32 \\
\midrule
Base & 0.234 & -- & -- & -- & -- & -- \\
Best-of-N (argmax) & 0.234 & 0.260 & 0.289 & 0.312 & 0.334 & 0.354 \\
Best-of-N (draw) & 0.234 & 0.260 & 0.288 & 0.309 & 0.329 & 0.347 \\
Chunk pruning & 0.234 & 0.289 & 0.342 & 0.371 & 0.398 & 0.415 \\
\method{}, dense (FK steering), draw & 0.234 & 0.256 & 0.316 & 0.350 & 0.370 & 0.387 \\
\method{}, dense (FK steering), argmax & 0.234 & -- & -- & -- & -- & 0.441 \\
\method{}, boundary (draw) & 0.234 & 0.258 & 0.305 & 0.332 & 0.366 & 0.385 \\
\method{}, boundary (argmax) & 0.234 & -- & -- & 0.356 & -- & 0.428 \\
\midrule
\method{} (draw), $\alpha = 0.05$ & 0.234 & 0.253 & 0.273 & 0.281 & 0.294 & 0.293 \\
\method{} (draw), $\alpha = 0.2$ & 0.234 & 0.236 & 0.254 & 0.247 & 0.254 & 0.251 \\
\bottomrule
\end{tabular}
\end{center}
\end{table}

\begin{table}[h]
\caption{Repetition (4M dance model), $\alpha = 0.02$; the additive and full forms coincide. A dash marks a budget at which the argmax sequences were not regenerated.}
\label{tab:rrep_full}
\begin{center}
\small
\begin{tabular}{l cccccc}
\toprule
$N$ & 1 & 2 & 4 & 8 & 16 & 32 \\
\midrule
Base & 0.718 & -- & -- & -- & -- & -- \\
Best-of-N (argmax) & 0.718 & 0.768 & 0.805 & 0.850 & 0.875 & 0.901 \\
Best-of-N (draw) & 0.718 & 0.768 & 0.804 & 0.849 & 0.871 & 0.895 \\
Chunk pruning & 0.718 & 0.761 & 0.792 & 0.811 & 0.835 & 0.887 \\
\method{}, dense (FK steering), draw & 0.718 & 0.768 & 0.836 & 0.876 & 0.901 & 0.921 \\
\method{}, dense (FK steering), argmax & 0.718 & -- & -- & 0.898 & -- & 0.962 \\
\method{}, boundary (draw) & 0.718 & 0.768 & 0.812 & 0.864 & 0.892 & 0.924 \\
\method{}, boundary (argmax) & 0.718 & -- & -- & 0.874 & -- & 0.947 \\
\bottomrule
\end{tabular}
\end{center}
\end{table}

\begin{table}[h]
\caption{Prompt adherence (terminal CLAP, music model), $\alpha = 0.0068$. \method{} uses the prefix-score potential; $\beta$ tempers the increments (Proposition~\ref{prop:prefix}). A dash marks a budget at which the argmax sequences were not regenerated.}
\label{tab:clap_full}
\begin{center}
\small
\begin{tabular}{l cccccc}
\toprule
$N$ & 1 & 2 & 4 & 8 & 16 & 32 \\
\midrule
Base & 0.470 & -- & -- & -- & -- & -- \\
Best-of-N (argmax) & 0.470 & 0.488 & 0.502 & 0.513 & 0.520 & 0.530 \\
Best-of-N (draw) & 0.470 & 0.488 & 0.502 & 0.513 & 0.519 & 0.527 \\
Chunk pruning & 0.470 & 0.499 & 0.523 & 0.536 & 0.543 & 0.551 \\
\method{}, dense (FK steering), draw & 0.470 & 0.488 & 0.520 & 0.538 & 0.542 & 0.546 \\
\method{}, dense (FK steering), argmax & 0.470 & -- & -- & 0.545 & -- & 0.559 \\
\method{} $\beta = 1$ (draw) & 0.470 & 0.488 & 0.518 & 0.534 & 0.538 & 0.547 \\
\method{} $\beta = 1$ (argmax) & 0.470 & -- & -- & 0.538 & -- & 0.560 \\
\method{} $\beta = 0.5$ (draw) & 0.470 & 0.488 & 0.513 & 0.526 & 0.533 & 0.544 \\
\method{} $\beta = 0.25$ (draw) & 0.470 & 0.488 & 0.507 & 0.518 & 0.525 & 0.533 \\
\bottomrule
\end{tabular}
\end{center}
\end{table}

\begin{table}[h]
\caption{Motif recurrence (music model), $\alpha = 0.019$; the prefix score is available at the last boundary.}
\label{tab:motif_full}
\begin{center}
\small
\begin{tabular}{l cccccc}
\toprule
$N$ & 1 & 2 & 4 & 8 & 16 & 32 \\
\midrule
Base & 0.814 & -- & -- & -- & -- & -- \\
Best-of-N (argmax) & 0.814 & 0.849 & 0.876 & 0.897 & 0.910 & 0.918 \\
Best-of-N (draw) & 0.814 & 0.849 & 0.874 & 0.892 & 0.901 & 0.904 \\
Chunk pruning & 0.814 & 0.848 & 0.870 & 0.894 & 0.906 & 0.918 \\
\method{}, dense (FK steering), draw & 0.814 & 0.849 & 0.875 & 0.895 & 0.898 & 0.908 \\
\method{}, dense (FK steering), argmax & 0.814 & 0.849 & 0.885 & 0.910 & 0.922 & 0.932 \\
\method{} $\beta = 1$ (draw) & 0.814 & 0.849 & 0.872 & 0.891 & 0.900 & 0.903 \\
\method{} $\beta = 1$ (argmax) & 0.814 & 0.849 & 0.878 & 0.902 & 0.918 & 0.929 \\
\bottomrule
\end{tabular}
\end{center}
\end{table}

\begin{table}[h]
\caption{Scale: reward at $N = 32$ on the 4M and 38M dance models, $\alpha = 0.02$, both return rules, with the paired difference of \method{} (argmax) against chunk pruning. Orderings at $N = 32$ on beat alignment are identical at both sizes.}
\label{tab:scale}
\begin{center}
\small
\setlength{\tabcolsep}{4pt}
\resizebox{\textwidth}{!}{\begin{tabular}{l cccc cccc}
\toprule
& \multicolumn{4}{c}{Beat alignment} & \multicolumn{4}{c}{Repetition} \\
& \multicolumn{2}{c}{4M} & \multicolumn{2}{c}{38M} & \multicolumn{2}{c}{4M} & \multicolumn{2}{c}{38M} \\
& draw & argmax & draw & argmax & draw & argmax & draw & argmax \\
\midrule
Base ($N=1$) & \multicolumn{2}{c}{0.234} & \multicolumn{2}{c}{0.225} & \multicolumn{2}{c}{0.718} & \multicolumn{2}{c}{0.654} \\
Best-of-N & 0.347 & 0.354 & 0.353 & 0.360 & 0.895 & 0.901 & 0.780 & 0.784 \\
Chunk pruning & -- & 0.415 & -- & 0.423 & -- & 0.887 & -- & 0.795 \\
\method{}, dense (FK steering) & 0.387 & 0.441 & 0.406 & 0.455 & 0.921 & 0.962 & 0.822 & 0.874 \\
\method{}, boundary & 0.385 & 0.428 & 0.394 & 0.438 & 0.924 & 0.947 & 0.798 & 0.823 \\
\midrule
\method{} (argmax) $-$ pruning & -- & +0.013 & -- & +0.015 & -- & +0.059 & -- & +0.028 \\
\bottomrule
\end{tabular}}
\end{center}
\end{table}

\clearpage
\section{Human Evaluation Protocol}
\label{app:human}

\begin{table}[h]
\caption{Human evaluation at $N = 32$: percentage of pair judgements preferring \method{} (prescribed schedule, argmax) over each comparator, with 95\% bootstrap intervals over the 40 conditions of each testbed; ties counted as one half. Each cell contains 40 distinct pairs and 200 judgements, five valid ratings per pair.}
\label{tab:human}
\begin{center}
\footnotesize
\begin{tabular}{l cc cc}
\toprule
& \multicolumn{2}{c}{Dance (beat alignment)} & \multicolumn{2}{c}{Music (prompt adherence)} \\
\method{} preferred over & follows beat & overall & matches prompt & overall \\
\midrule
Base model & 73.5 [65.7, 81.2] & 66.0 [57.8, 74.5] & 77.0 [69.3, 84.2] & 69.0 [60.0, 77.3] \\
Best-of-N (argmax) & 66.0 [58.0, 74.0] & 60.5 [53.0, 68.0] & 69.0 [59.5, 77.5] & 62.0 [54.0, 69.8] \\
Chunk pruning & 69.0 [57.8, 78.2] & 63.0 [53.5, 72.0] & 71.5 [62.0, 79.8] & 65.0 [56.2, 73.5] \\
\bottomrule
\end{tabular}
\end{center}
\end{table}

\textbf{Stimuli.} For each of the 40 conditions of each testbed and generation seed 2000, the sequence returned by \method{} (prescribed schedule, argmax, $N = 32$, $\alpha$ as in Section~\ref{sec:setup}) was paired with the sequence returned by the base model, by best-of-N argmax and by chunk pruning ($M = N/4$) at the same budget, giving 40 pairs per comparator, 120 pairs per testbed and 240 pairs in total. Dance sequences were rendered as SMPL \citep{loper2015smpl} mesh videos at 30 fps with a fixed camera and the conditioning music as the soundtrack; music clips were the 40-second decodes, loudness-normalized to $-14$ LUFS. Within a pair both sequences share the condition and the seed, so raters heard the same music for both dances or read the same description for both clips, and method identities were hidden throughout.

\textbf{Raters.} Participants were recruited on Prolific and were required to be fluent in English, to have an approval rate of at least 95 percent and at least 100 completed studies, and to pass a headphone screening check. Each participant judged 20 substantive pairs from one testbed in a session of about 40 minutes and received US\$12, a target rate of US\$18 per hour. Two instructed-response attention checks were interleaved with the substantive pairs, and participants failing both were excluded: 4 of 64, leaving 60 retained participants, 30 per testbed, with replacement sessions restoring five valid ratings per pair. The retained data comprise 1{,}200 pair evaluations and 2{,}400 answers across the two questions. Pair order and the side assigned to \method{} were randomized independently per participant.

\textbf{Questions.} For each pair raters answered, in order: (1) \emph{Which dance follows the beat of the music better?} for dance, or \emph{Which clip better matches the description shown?} for music; (2) \emph{Which would you rather watch?} for dance, or \emph{Which would you rather listen to?} for music. Each question offered A, B and no preference. Question (1) reads the dimension the steering reward targets; question (2) reads overall preference.

\textbf{Statistics.} For each testbed, comparator and question, the win rate is the percentage of judgements preferring \method{}, with no-preference counted as one half, so that each table entry summarizes 200 judgements over 40 pairs. Intervals are 95\% percentile bootstraps over 1{,}000 resamples of the 40 conditions with replacement, keeping all judgements of a sampled condition together and using the same resampled indices across comparators and questions within a testbed; they quantify uncertainty over conditions and are reported without adjustment for multiple comparisons. Agreement is Fleiss' $\kappa$ per testbed and question over its 120 pairs after mapping the randomized A/B answers to \method{}, comparator and no preference. To relate the automatic reward to human preference, Spearman's $\rho$ across the 40 conditions pairs each condition's reward-question win rate with the automatic reward difference between \method{} and the same comparator.

\textbf{Results.} Table~\ref{tab:human} shows a consistent preference for \method{} on both questions. Against the base model, best-of-N and chunk pruning respectively, the reward-question win rate exceeds 50 percent in 34, 28 and 26 of the 40 dance conditions and in 34, 29 and 28 of the 40 music conditions; for overall preference the counts are 28, 26 and 26 for dance and 28, 25 and 27 for music. No-preference answers make up 12.5 percent of dance judgements and 12.0 percent of music judgements. Fleiss' $\kappa$ is $0.221$ for dance beat alignment, $0.122$ for dance overall preference, $0.191$ for music prompt adherence and $0.178$ for music overall preference. Spearman's $\rho$ between reward difference and reward-question win rate is positive for all six comparisons: $0.469$, $0.624$ and $0.580$ on dance and $0.225$, $0.525$ and $0.667$ on music against the base model, best-of-N and chunk pruning.

\end{document}